\documentclass{article}

\usepackage[nonatbib, final]{neurips_2020}

\usepackage[utf8]{inputenc} % allow utf-8 input
\usepackage[T1]{fontenc}    % use 8-bit T1 fonts
\usepackage{hyperref}       % hyperlinks
\usepackage{url}            % simple URL typesetting
\usepackage{booktabs}       % professional-quality tables
\usepackage{amsfonts}       % blackboard math symbols
\usepackage{nicefrac}       % compact symbols for 1/2, etc.
\usepackage{microtype}      % microtypography
\usepackage{bm,bbm,amsmath,amsfonts,amssymb}
\usepackage{xspace}
\usepackage{xcolor}
\usepackage{eucal}
\usepackage[alpha]{mdpn}
\usepackage{appendix}

\usepackage{url}            % simple URL typesetting
\usepackage{booktabs}       % professional-quality tables
\usepackage{amsfonts}       % blackboard math symbols
\usepackage{nicefrac}       % compact symbols for 1/2, etc.
\usepackage{microtype}      % microtypography
\usepackage{amsmath}
\usepackage{mathtools}  % \coloneqq

\usepackage{algorithm}
\usepackage[ruled,vlined,algo2e]{algorithm2e}
\usepackage{hyperref}

\usepackage[numbers]{natbib}

\usepackage{caption}
\usepackage{subcaption}
\usepackage{wrapfig}
\usepackage{amsthm}	% Theorems
\newtheorem{lemma}{Lemma}[section]
\usepackage{theoremref} % for \theoremref to nicely reference theorems
\usepackage{thmtools}
\usepackage{thm-restate} % to restate the theorem
\usepackage{amsxtra} % for \mathfrak{c} = cardinality of continuum

\usepackage{textcase}
\DeclareMathOperator*{\argmax}{argmax}

\DeclareMathOperator*{\argmin}{argmin}

\newcommand{\msim}{\mathcal{M}_{s}}
\newcommand{\mreal}{\mathcal{M}_{t}}

\DeclareMathOperator{\defd}{:=}

\newcommand{\xset}{\ensuremath{\mathcal{X}}}    % state space

\newcommand{\gaifo}{\textsc{gai}f\textsc{o}}
\newcommand{\garat}{\textsc{garat}}
\newcommand{\gail}{\textsc{gail}}
\newcommand{\gat}{\textsc{gat}}
\newcommand{\ifo}{\textsc{i}f\textsc{o}}
\newcommand{\rarl}{\textsc{rarl}}

\setcitestyle{square}

\title{An Imitation from Observation Approach to Transfer Learning with Dynamics Mismatch}

\author{%
  Siddharth Desai\textsuperscript{\textsection} \\
  Department of Mechanical Engineering\\
  The University of Texas at Austin\\
  \texttt{sidrdesai@utexas.edu} \\
  \And
  Ishan Durugkar \textsuperscript{\textsection} \\
  Department of Computer Science\\
  The University of Texas at Austin\\
  \texttt{ishand@cs.utexas.edu}
  \And
  Haresh Karnan \textsuperscript{\textsection}\\
  Department of Mechanical Engineering\\
  The University of Texas at Austin\\
  \texttt{haresh.miriyala@utexas.edu}
   \And
   Garrett Warnell \\
   Army Research Laboratory \\
   \texttt{garrett.a.warnell.civ@mail.mil} \\
   \And
   Josiah P. Hanna \thanks{to be joining the Computer Sciences department at the University of Wisconsin -- Madison} \\
   School of Informatics \\
   The University of Edinburgh \\
   \texttt{josiah.hanna@ed.ac.uk} \\
   \And
   Peter Stone \\
   Department of Computer Science\\
   The University of Texas at Austin\\
   and Sony AI \\
   \texttt{pstone@cs.utexas.edu} \\
}

\begin{document}

% author names and affiliations
% use a multiple column layout for up to three different
% affiliations
% \author{\IEEEauthorblockN{Siddharth Desai\IEEEauthorrefmark{1}\textsuperscript{\textsection}, Ishan Durugkar\IEEEauthorrefmark{1}\textsuperscript{\textsection}, Haresh Karnan\IEEEauthorrefmark{1}\textsuperscript{\textsection}, Garrett Warnell\IEEEauthorrefmark{2}, Josiah Hanna\IEEEauthorrefmark{3}, Peter Stone\IEEEauthorrefmark{1}\IEEEauthorrefmark{4}}
% \IEEEauthorblockA{\IEEEauthorrefmark{1} University of Texas at Austin, 
% \IEEEauthorrefmark{2} Army Research Lab, 
% \IEEEauthorrefmark{3} University of Edinburgh, 
% \IEEEauthorrefmark{4} Sony AI \\
% sidrdesai@utexas.edu, ishand@cs.utexas.edu, haresh.miriyala@utexas.edu\\ garrett.a.warnell.civ@mail.mil, josiah.hanna@ed.ac.uk, pstone@cs.utexas.edu}
% % \and
% % \IEEEauthorblockN{Garrett Warnell}
% % \IEEEauthorblockA{Army Research Lab\\
% % warnellg@cs.utexas.edu}
% % \and
% % \IEEEauthorblockN{Josiah Hanna}
% % \IEEEauthorblockA{University of Edinburgh\\
% % josiah,hanna@utexas.edu}
% }

\maketitle

\begingroup\renewcommand\thefootnote{\textsection}
\footnotetext{Equal contribution}

\begin{abstract}
    We examine the problem of transferring a policy learned in a source environment to a target environment with different dynamics, particularly in the case where it is critical to reduce the amount of interaction with the target environment during learning. This problem is particularly important in sim-to-real transfer because simulators inevitably model real-world dynamics imperfectly.
    % , e.g., when trying to leverage a simulator to learn a policy that works well in the real world. 
    In this paper, we show that one existing solution to this transfer problem---{\em grounded action transformation}---is closely related to the problem of {\em imitation from observation} (IfO):
    %
    % We examine the problem of learning a policy in a source environment and transferring it to a target environment with different dynamics.
    % Such a transfer can be useful in scenarios like using a simulator to learn a policy that works well in the real world.
    % % The sim-to-real transfer problem deals with leveraging large amounts of inexpensive simulation experience to help artificial agents learn behaviors intended for the real world more efficiently.
    % To be useful for such sim-to-real transfer, we want to limit the interaction with the target environment during learning.
    % One existing approach uses a few interactions with the target environment to make the dynamics of the source more like the target, called grounded transfer.
    % In this paper, we show that a particular grounded transfer approach, grounded action transformation, is closely related to the problem of imitation from observation (IfO): 
    learning behaviors that mimic the observations of behavior demonstrations.
    After establishing this relationship, we hypothesize that recent state-of-the-art approaches from the IfO literature can be effectively repurposed for grounded transfer learning. To validate our hypothesis we derive a new algorithm---generative adversarial reinforced action transformation (\garat{})---based on adversarial imitation from observation techniques. We run experiments in several domains with mismatched dynamics, and find that agents trained with \garat{} achieve higher returns in the target environment compared to existing black-box transfer methods.
\end{abstract}

\section{Introduction}

Transfer learning with dynamics mismatch refers to using experience in a source environment to more efficiently learn control policies that perform well in a target environment, where the two environments differ only in their transition dynamics.
For example, if the friction coefficient in the source and target environments is sufficiently different it might cause the action of placing a foot on the ground to work well in one environment, but cause the foot to slip in the other.
One possible application of such transfer is where the source environment is a simulator and the target environment is a robot in the real world, called sim-to-real.
In sim-to-real scenarios, source environment (simulator) experience is readily available, but target environment (real world) experience is expensive.
%
% Though we only evaluate in virtual environments in this work, we use this metric to compare the performance of different algorithms.
% In the robot learning community, {\em sim-to-real} approaches seek to leverage inexpensive simulation experience to more efficiently learn control policies that perform well in the real world.
%
%This paradigm allows us to utilize powerful machine learning techniques without extensive testing in the target environment.
% , which can be expensive, time-consuming, and potentially dangerous.
%
Sim-to-real transfer has been used effectively to learn a fast humanoid walk \cite{hanna2017grounded}, dexterous manipulation \cite{openai2019solving, R21, R22, R23, R24, R25, R26}, and agile locomotion skills \cite{peng2020learning}.
In this work, we focus on the paradigm of simulator grounding \cite{farchy2013gsl, hanna2017grounded, chebotar2019closing}, which modifies the source environment's dynamics to more closely match the target environment dynamics using a relatively small amount of target environment data.
Policies then learned in such a grounded source environment transfer better to the target environment.

Separately, the machine learning community has also devoted attention to imitation learning \cite{bakker1996robot}, i.e.\ the problem of learning a policy to mimic demonstrations provided by another agent.
%
% An often unstated requirement of imitation learning is to learn from few demonstrations, since expert demonstrations are expensive to procure.
%
In particular, recent work has considered the specific problem of {\em imitation from observation} (IfO) \cite{liu2018imitation}, in which an imitator mimics the expert's behavior without knowing which actions the expert took, only the outcomes of those actions (i.e.\ state-only demonstrations).
While the lack of action information presents an additional challenge, recently-proposed approaches have suggested that this challenge may be addressable \cite{torabi2018behavioral, Torabi_2019}.

In this paper, we show that a particular grounded transfer technique that has been shown to successfully accomplish sim-to-real transfer, called {\em grounded action transformation} (\gat{}) \cite{hanna2017grounded}, can be seen as a form of IfO.
We therefore hypothesize that recent, state-of-the-art approaches for addressing the IfO problem might also be effective for grounding the source environment, leading to improved transfer.
Specifically, we derive a distribution-matching objective similar to ones used in adversarial approaches for generative modeling \cite{goodfellow2014generative}, imitation learning \cite{ho2016gail}, and IfO \cite{torabi2018generative} with considerable empirical success.
Based on this objective, we propose a novel algorithm, Generative Adversarial Reinforced Action Transformation (\garat{}), to ground the source environment by reducing the distribution mismatch between the source and target environments.

Our experiments confirm our hypothesis by showing that \garat{} reduces the difference in the dynamics between two environments more effectively than \gat{}.
Moreover, our experiments show that, in several domains, this improved grounding translates to better transfer of policies from one environment to the other.

The contributions of this paper are as follows: {\em (1)} we show that learning the grounded action transformation can be seen as an \emph{IfO} problem, {\em (2)} we derive a novel adversarial imitation learning algorithm, \garat{}, to learn an action transformation policy for transfer learning with dynamics mismatch, and {\em (3)} we experimentally evaluate the efficacy of \garat{} for transfer with dynamics mismatch.

\section{Background}
We begin by introducing notation, reviewing the transfer learning with dynamics mismatch problem formulation, and describing the action transformation approach for sim-to-real transfer.
We also provide a brief overview of imitation learning and imitation from observation.

\subsection{Notation}
We consider here sequential decision processes formulated as Markov decision processes (MDPs) \cite{sutton2018reinforcement}. An MDP $\mathcal{M}$ is a tuple $\langle \sset, \aset, \Rfun, \Tfun, \D, \rho_0 \rangle$ consisting of a set of states, $\sset$; a set of actions, $\aset$; a reward function, $\Rfun: \sset \times \aset \times \sset \longmapsto \Delta([\rmin, \rmax])$  (where $\Delta([\rmin, \rmax])$ denotes a distribution over the interval $[\rmin, \rmax] \subset \mathbb{R}$);
% \ishan{can we write $\Rfun$ as a random variable? Perhaps $\Rfun: \sset \times \aset \times \sset \times \rset \longmapsto [0, 1]$?}
a discount factor, $\D \in [0, 1)$; a transition function, $\Tdef$; and an initial state distribution, $\rho_0: \Delta(\sset)$.
An RL agent uses a policy $\p: \sset \longmapsto \Delta(\aset)$ to select actions in the environment.
In an environment with transition function $\Tfun \in \Tset$, the agent aims to learn a policy $\p \in \pset$ to maximize its expected discounted return  $\mathbb{E}_{\p, \Tfun} [G_0] = \mathbb{E}_{\p, \Tfun} \left[ \sum_{t=0}^{\infty} \D^{t} \rt{t} \right]$, where $\rt{t} \sim \Rfun(s_t, a_t, s_{t+1})$, $s_{t+1} \sim \Tfun(s_t, a_t)$, $a_t \sim \p(s_t)$, and $s_0 \sim \rho_0$.

Given a fixed $\pi$ and a specific transition function $\Tfun_q$, the marginal transition distribution is $\rho_q(s, a, s') \defd (1 - \D) \p(a|s) \Tfun_q(s'|s, a) \sum_{t=0}^\infty \D^t p(s_t=s|\p, \Tfun_q)$ where
$p(s_t=s|\p, \Tfun_q)$ is the probability of being in state $s$ at time $t$.
% \garrett{there's a little bit of notational aliasing going on here -- above, $P$ is the transition function of the MDP, but then right here I see $P$ being used as a general probability, right?}
% \ishan{Oh shoot! Yes. Hmm.. Should I use $p$ for probability?}
% \garrett{Sure, that could work. Whatever you decide, just make sure you're consistent with it throughout the document.}
The marginal transition distribution is the probability of being in state $s$ marginalized over time $t$, taking action $a$ under policy $\p$, and ending up in state $s'$ under transition function $\Tfun_q$ (laid out more explicitly in Appendix \ref{app:marginal}).
%
% We do not emphasize its dependence on $\p$ since our focus is on the transition function.
%
% \josiah{Why is the $(1-\gamma)$ in the marginal transition definition? Is it the probability of the transition and then terminating after that transition?}
% \ishan{It's the normalizing factor, to make sure $\rho$ sums to 1.
% I know that the marginal distribution is usually denoted as unnormalized though.
% Removing $(1 - \D)$ here would modify Equation \ref{eqn:marg_return} below, and I might have to adjust a proof.}
We can denote the expected return under a policy $\p$ and a transition function $\Tfun_q$ in terms of this marginal distribution as:
\begin{align} \label{eqn:marg_return}
    \mathbb{E}_{\p, q} [ G_0 ] = \frac{1}{(1- \D)}\sum_{s, a, s'} \rho_{q}(s, a, s') \Rfun(s'|s, a)
\end{align}

\subsection{Transfer Learning with Dynamics Mismatch and Grounded Action Transformation} \label{sec:s2r}

Let $\Tfun_{s}, \Tfun_{t} \in \Tset$ be the transition functions for two otherwise identical MDPs, $\msim$ and $\mreal$, representing the source and target environments respectively.
Transfer learning with dynamics mismatch, as opposed to transfer learning in general, aims to train an agent policy to maximize return in $\mreal$ with limited trajectories from $\mreal$, and as many as needed in $\msim$.
% One approach to sim-to-real transfer is to train on a diverse range of dynamics by introducing domain randomization \cite{tobin2017domain} or adversarial perturbations \cite{pinto2017robust}, without \emph{any} real world interactions required. \ishan{Since we talk about domain randomization in the related work, can we remove this previous sentence?} \haresh{Yeah, lets focus on grounded sim-to-real methods here. "Simulator grounding approach to sim-to-real focuses on using.... "}
%
% One way to deal with sim-to-real transfer is to make the agent policy robust to changes in environment dynamics, by domain randomization  or injecting noise \cite{pinto2017robust, jakobi1997evolutionary}.
%
% \textsc{rarl} \cite{pinto2017robust}, for example, learns an adversary that introduces perturbations in the simulator while the agent is learning to make the policy more robust.

The work presented here is specifically concerned with a particular class of approaches used in sim-to-real transfer known as simulator grounding approaches \cite{allevato2019tunenet, chebotar2019closing,  farchy2013gsl}.
Here the source environment is the simulator and the target environment is the real world.
These approaches use some interactions with the target environment to \emph{ground} the source environment dynamics to more closely match the target environment dynamics.
Because it may sometimes be difficult or impossible to modify the source environment itself (when the source environment is a black-box simulator, for example), the recently-proposed grounded action transformation (\gat{}) approach \cite{hanna2017grounded} seeks to instead induce grounding by modifying the agent's actions before using them in the source environment.
This modification is accomplished via an action transformation function $\p_{g}: \sset \times \aset  \longmapsto \Delta(\aset)$ that takes as input the state and action of the agent, and produces an action to be presented to the source environment.
From the agent's perspective, composing the action transformation with the source environment changes the source environment's transition function.
We call this modified source environment the {\em grounded} environment, and its transition function is given by
\begin{align} \label{eqn:trans_g}
\Tfun_{g}(s'|s, a) &= \sum_{\Tilde{a} \in \aset} \Tfun_{s}(s'|s, \Tilde{a})\p_{g}(\Tilde{a}|s, a) 
\end{align}
The action transformation approach aims to learn function $\p_g \in \pset_g$ such that the resulting transition function $\Tfun_{g}$ is as close as possible to $\Tfun_{t}$.
We denote the marginal transition distributions in the source and target environments by $\rho_{s}$ and $\rho_{t}$ respectively, and $\rho_g \in \mathcal{P}_{g}$ for the grounded environment.
\gat{} learns a model of the target environment dynamics $\hat{\Tfun}_{t}(s'|s, a)$, an inverse model of the source environment dynamics $\hat{\Tfun}_{s}^{-1}(a|s, s')$, and uses the composition of the two as the action transformation function, i.e.\ $\p_g(\Tilde{a}|s, a) = \hat{\Tfun}_{s}^{-1}(\Tilde{a}|s, \hat{\Tfun}_{t}(s'|s, a))$.
% \ishan{Should we use some modification here (like a $\hat{\cdot}$) to show that the models are estimates?} \haresh{yeah, i think we should. immediately before you talk about GAT, you've also used $P_{sim}$ which is not a model.}
%
%This action transformation function is then applied to the actions of the agent in the simulator so that the effect is similar to the real world.\sid{Removed this line because it seemed repetitive.}

% \begin{align*}
%     \rho_{t, at}(s) &:= \sum_{s_{t-1} \in \sset}\rho_{t-1, at}(s_{t-1}) \sum_{a_{t-1} \in \aset} \p(a_{t-1}|\st{t-1}) \sum_{\Tilde{a}_{t-1} \in \aset} \p_{g}(\Tilde{a}_{t-1}|s_{t-1}, a_{t-1}) \Tfun_{sim}(s|s_{t-1}, \Tilde{a}_{t-1}) \\
%     \rho_{g}(s) &:= \sum_{t=0}^T \rho_{t, at}(s)
%     % \rho_{g}(s, a) &:= \rho_{g}(s) \p(a|s)
% \end{align*}

% Once trained in the simulator with this \emph{Grounded} Action Transformer, the agent policy $\p$ can be applied directly in the real world with the expectation of performing just as well.
% This approach can be seen as attempting to imitate the transitions in the real world.
% \ishan{Need to rephrase the sentences above}
% \josiah{Can you foreshadow that GAT is imitation learning?}

\subsection{Imitation Learning} \label{sec:il}
% \garrett{This section could start with a little more to lead the reader to the conclusion you'd like them to draw in the end. Start maybe with something like "In parallel to advances in sim-to-real, the machine learning community has also made considerable recent progress in the problem of imitation learning."}
In parallel to advances in sim-to-real transfer, the machine learning community has also made considerable progress on the problem of imitation learning.
Imitation learning \cite{bakker1996robot, ross2011reduction, schaal1997learning} is the problem setting where an agent tries to mimic trajectories $\{\xi_0, \xi_1, \ldots\}$ where each $\xi$ is a demonstrated trajectory $\{(s_0, a_0), (s_1, a_1), \ldots\}$ induced by an expert policy $\p_{exp}$. 

Various methods have been proposed to address the imitation learning problem.
Behavioral cloning \cite{bain1995framework} uses the expert's trajectories as labeled data and uses supervised learning to recover the maximum likelihood policy.
% to have generated the data.
%
Another approach instead relies on reinforcement learning to learn the policy, where the required reward function is recovered using inverse reinforcement learning (IRL) \cite{ng2000algorithms}.
IRL aims to recover a reward function under which the demonstrated trajectories would be optimal.
A related setting to learning from state-action demonstrations is the imitation from observation (IfO) \cite{liu2018imitation, pavse2019ridm, torabi2018behavioral, torabi2018generative} problem.
Here, an agent observes an expert's state-only trajectories $\{\zeta_0, \zeta_1, \ldots\}$ where each $\zeta$ is a sequence of states $\{s_0, s_1, \ldots\}$.
The agent must then learn a policy $\p(a|s)$ to imitate the expert's behavior, without being given labels of which actions to take.

\section{\gat{} as Imitation from Observation} \label{sec:s2r_ifo}
We now show that the underlying problem of \gat{}---i.e., learning an action transformation for sim-to-real transfer---can also been seen as an IfO problem.
% We first identify the defining characteristics of an IfO problem, and then show how learning an action transformation function with \gat{} has those same characteristics.
%
Adapting the definition by \citet{liu2018imitation}, an IfO problem is a sequential decision-making problem
% that requires learning a policy (i.e.\ a mapping from states to action distributions),
where the policy imitates state-only trajectories $\{\zeta_0, \zeta_1, \ldots\}$ produced by a Markov process, with no information about what actions generated those trajectories.
To show that the action transformation learning problem fits this definition, we must show that it {\em (1)} is a sequential decision-making problem and {\em (2)} aims to imitate state-only trajectories produced by a Markov process, with no information about what actions generated those trajectories.
% the action transformation function {\em (1)} is faced with a sequential decision-making problem, {\em (2)} is a policy, and {\em (3)} aims to imitate state-only trajectories, with no information about what actions generated those trajectories.
%

Starting with {\em (1)}, it is sufficient to show that the action transformation function is a policy in an MDP \cite{puterman1990markov}.
This action transformation MDP can be seen clearly if we combine the target environment MDP and the fixed agent policy $\p$.
Let the joint state and action space $\xset := \sset \times \aset$ with $x \defd (s, a) \in \xset$ be the state space of this new MDP.
The combined transition function is $\Tfun^x_{s}(x'|x, \Tilde{a}) = \Tfun_{s}(s'|s, \Tilde{a}) \p(a'|s')$, where $x'=(s',a')$, and initial state distribution is $\rho^x_0(x) = \rho_0(s)\p(a|s)$.
For completeness, we consider a reward function $R^x: \xset \times \aset \times \xset \longmapsto \Delta([\rmin, \rmax])$ and discount factor $\gamma_x \in [0, 1)$, which are not essential for an IfO problem.
With these components, the action transformation environment is an MDP $\langle\xset, \aset, R^{x}, \Tfun^x_{s}, \gamma_x, \rho_0^x \rangle$.
The action transformation function $\p_g(\Tilde{a}|s, a)$, now $\p^x_g(\Tilde{a}|x)$, is then clearly a mapping from states to a distribution over actions, i.e.\ it is a policy in an MDP.
%
% This function ($\p_g: \xset \longmapsto \Delta(\aset)$) is thus a mapping from states to a distribution over actions, i.e.\ it is a policy.
%
Thus, the action transformation learning problem is a sequential decision-making problem.

We now consider the action transformation objective to show {\em (2)}.
When learning the action transformation policy, we have trajectories $\{\tau_0, \tau_1, \ldots\}$, where each trajectory $\tau=\{(s_0, a_0 \sim \p(s_0)), (s_1, a_1 \sim \p(s_1)), \ldots\}$ is obtained by sampling actions from agent policy $\p$ in the target environment.
Re-writing $\tau$ in the above MDP, $\tau=\{x_0, x_1, \ldots\}$.
If an expert action transformation policy $\p^*_g \in \pset_g$ is capable of mimicking the dynamics of the target environment, $\Tfun^x_{t}(x'|x) = \sum_{\Tilde{a} \in \aset}\Tfun^x_{s}(x'|x, \Tilde{a}) \p^*_{g}(\Tilde{a}|x)$,
then we can consider the above trajectories to be produced by a Markov process with dynamics $\Tfun^x_{s}(x'|x, \Tilde{a})$ and policy $\p^*_g(\Tilde{a}|x)$.
%
%The action transformation policy does not have access to the actions of this expert policy.
%
The action transformation aims to imitate the state-only trajectories $\{\tau_0, \tau_1, \ldots\}$ produced by a Markov process, with no information about what actions generated those trajectories.
% needs to learn which actions to take to mimic the state-only trajectories $\{\tau_0, \tau_1, \ldots\}$ produced by the expert, without explicit %information about what actions it should take to generate those trajectories.
% access to the actions of the expert policy.

The problem of learning the action transformation thus satisfies the conditions we identified above, and so it is an IfO problem.

\section{Generative Adversarial Reinforced Action Transformation} \label{sec:adversarial grounding}

\begin{algorithm}[b]
\caption{GARAT}
\label{alg}
% \begin{algorithmic}[1]
% \end{algorithmic}
\SetAlgoLined
\KwIn{Target environment with $\Tfun_{t}$, source environment with $\Tfun_{s}$, number of update steps $N$}\;
Agent policy $\p$ with parameters $\eta$\ , pretrained in source environment; \\
Initialize action transformation policy $\p_g$ with parameters $\theta$\; \\
Initialize discriminator $D_{\phi}$ with parameters $\phi$\; \\
\While{performance of policy $\p$ in target environment not satisfactory}{
    Rollout policy $\p$ in target environment to obtain trajectories $\{\tau_{t, 1}, \tau_{t, 2}, \ldots\}$\; \\  % =\{s_0, a_0, s_1,\ldots\}
    \For{$i=0,1,2, \ldots N$}{
    Rollout Policy $\pi$ in grounded source environment and obtain trajectories $\{\tau_{g, 1}, \tau_{g, 2}, \ldots\}$\; \\ % =\{s_0, a_0, s_1,\ldots\}
    Update parameters $\phi$ of $D_{\phi}$ using gradient descent to minimize \\ $-\left(\mathbb{E}_{\tau_{g}}[\log (D_{\phi}(s, a, s'))] + \mathbb{E}_{\tau_{t}}[ \log(1 - D_{\phi}(s, a, s'))\right)$\; \\
    Update parameters $\theta$ of $\p_g$ using policy gradient with reward $-[\log D_{\phi}(s, a, s')]$\;
    }
    Optimize parameters $\eta$ of $\p$ in source environment grounded with action transformer $\p_g$\;
}
\end{algorithm}

% \garrett{I'd rather see this section organized as follows:
% \begin{enumerate}
%     \item Paragraph that provides an intuitive motivation/explanation of what we hope/need to accomplish theoretically in this section. In my opinion, the current first paragraph doesn't do this very well. You need to tell the reader very plainly and simply why its worth it for them to read this heavy section.
%     \item Statement of the main theoretical result as a theorem (i.e., "Theorem 1"). Is it currently called Proposition 4.3?
%     \item Paragraph sketching how we will prove the main result at a high level.
%     \item Using the sketch paragraph as a guide, go through sub-sections (lemmas, propositions, etc.) roughly as currently organized.
%     \item Finally, a paragraph called "Proof of Theorem 1" or something that references the lemmas, propositions in order to convince the reader that Theorem 1 is true.
% \end{enumerate}
% Importantly, Theorem 1 is stated and proven in the body of the manuscript. Sure, we'll rely on lemmas and propositions that are themselves only proven in the appendix, but that's OK.
% }

% Section \ref{sec:s2r_ifo} showed that the action transformation environment is an MDP, and that the problem of learning an action transformation policy is an IfO problem.
%
The insight above naturally leads to the following question: {\em if learning an action transformation for transfer learning is equivalent to IfO, might recently-proposed IfO approaches lead to better transfer learning approaches?}
To investigate the answer, we derive a novel generative adversarial approach inspired by \gaifo{}\cite{torabi2018generative} that can be used to train the action transformation policy using IfO.
A source environment grounded with this action transformation policy can then be used to train an agent policy which can be expected to transfer effectively to a given target environment.
We call our approach generative adversarial reinforced action transformation (\garat{}), and Algorithm \ref{alg} lays out its details.

The rest of this section details our derivation of the objective used in \garat{}.
% In the rest of this section, we present our derivation of the adversarial objective used in \garat{}.
%
First, in Section \ref{sec:atirl}, we formulate a procedure for action transformation using a computationally expensive IRL step to extract a reward function and then learning an action transformation policy based on that reward.
%
% The intermediate IRL step in this procedure is computationally expensive.
%
Then, in Section \ref{sec:theory}, we show that this entire procedure is equivalent to directly reducing the marginal transition distribution discrepancy between the target environment and the grounded source environment.
This is important, as recent work \cite{goodfellow2014generative, ho2016gail,torabi2018generative} has shown that adversarial approaches are a promising algorithmic paradigm to reduce such discrepancies.
Thus, in Section \ref{sec:garat}, we explicitly formulate a generative adversarial objective upon which we build the proposed approach.

\subsection{Action Transformation Inverse Reinforcement Learning}
\label{sec:atirl}

%  \haresh{this is our main idea in this paper. would be nice if we can bold this, or put it in a box if NeurIPS format allows.}
% \ishan{We'll move this part to the intro and find a way to highlight it}

% The agent now has to learn an action transformation policy $\p_{g}$ such that the \emph{effect} of a separate policy $\p$ will be the same under $\Tfun_{g}$ and $\Tfun_{real}$.
% We have shown in Section \ref{sec:IfO_pers} that the setting where the AT can observe real world transitions $(s, a, s', a')$ induced by an agent policy $\p$ and transition function $\Tfun_{real}$, and needs to learn an action transformer policy $\p_{g}$ that induces the same transitions is the same as the IfO framework where the agent can only observe state transitions $(s, s')$ and needs to learn a policy to induce the same  state transitions \cite{torabi2018generative,torabi2018behavioral}.
% \ishan{that's a big sentence.}\sid{and it's a fragment.}

% Section \ref{sec:s2r_ifo} shows that the Sim-to-Real problem can be characterized as an \emph{IfO} problem.
% Indeed, the Behavioral Cloning from Observations (BCO) \cite{torabi2018behavioral} approach of learning a backward model to infer which action could have led to the observed expert transition is analogous to the backward model used in \gat{} to infer how the action should be transformed in order for the transition to be more realistic.

We first lay out a procedure to learn the action transformation policy by extracting the appropriate cost function, which we term action transformation IRL (ATIRL).
We use the cost function formulation in our derivation, similar to previous work \cite{ho2016gail, torabi2018generative}.
ATIRL aims to identify a cost function such that the observed target environment transitions yield higher return than any other possible transitions.
We consider the set of cost functions $\mathcal{C}$ as all functions $\mathbb{R}^{\sset \times \aset \times \sset} = \{c: \sset \times \aset \times \sset \longmapsto \mathbb{R}\}$.

\begin{equation}
    \mathtt{ATIRL}_{\psi}(\Tfun_{t}) \defd \argmax_{c \in \mathcal{C}} -\psi(c) + \left(\min_{\p_{g} \in \pset_{g}} \mathbb{E}_{\rho_g}[c(s, a, s')] \right) - \mathbb{E}_{\rho_{t}}[c(s, a, s')]
    \label{eqn:ATIRL}
\end{equation}
% \commentp{It looks like there are expectations taken over the transition functions (e.g. $\Tfun_{real}$).  What does that mean?  Isn't there just a single real and a single real transition function?  And similarly, the grounded transition function is determined by $\p_g$, so why is there an expectation there?  I guess it's just not clear to me what the expectation is over.}
% \ishan{The expectation is over transitions due to agent policy $\p$ and the transition function of the grounded simulator and the real world, respectively. I've removed the dependence on agent policy $\p$ everywhere in this section because the notation was getting very dense and the agent policy is assumed to be fixed.}
% \commentp{I think that needs to be said.  Actually, rather than expectation, isn't it actually summed over observed transitions?  Whichever it is, it needs to be clarified.}
% \ishan{I realized I didn't understand what you were asking here. The expectation is taken over transitions, but for a particular transition function}
where $\psi: \mathbb{R}^{\sset \times \aset \times \sset} \longmapsto \overline{\mathbb{R}}$ is a (closed, proper) convex reward function regularizer, and $\overline{\mathbb{R}}$ denotes the extended real numbers $\mathbb{R} \bigcup \{\infty\}$.
This regularizer is used to avoid overfitting the expressive set $\mathcal{C}$.
Note that $\p_g$ influences $\rho_g$ (Equation \ref{eqn:rho_g} in Appendix \ref{app:marginal}) and $\Tfun_{t}$ influences $\rho_{t}$.
Similar to \gaifo{}, we do not use causal entropy in our ATIRL objective due to the surjective mapping from $\pset_{g}$ to $\mathcal{P}_{g}$.
%
% Maximum entropy IRL \cite{ziebart2008maximum} also contains a causal entropy term which is not incorporated here since we cannot guarantee an injective mapping from $\pset_{g}$ to $\mathcal{P}^{g}$.
% %
% However, not being able to include the causal entropy only negates the existence of a unique cost function obtained through this objective, and causal entropy can still be utilized to choose between these cost functions.

The action transformation then uses this per-step cost function as a reward function in an RL procedure: $\mathtt{RL}(c) \defd \argmin_{\p_g \in \pset_{g}} \mathbb{E}_{\rho_{g}} [c(s, a, s')]$.
%
% This RL procedure $\mathtt{RL}(c) \defd \argmin_{\p_g \in \pset_{g}} \mathbb{E}_{\Tfun_{g}} [c(s, a, s')]$ extracts the optimal action transformer policy with respect to the above cost function, and should lead to transitions that imitate those the real world's.
%
% \josiah{This seems like a strong assumption. Should we acknowledge this and point out it is just for theory?}
We assume here for simplicity that there is an action transformation policy that can mimic the target environment dynamics perfectly.
That is, there exists a policy $\p_{g} \in \pset_{g}$, such that $\Tfun_{g}(s'|s, a) = \Tfun_{t}(s'|s, a) \forall s \in \sset, a \in \aset$. 
% \josiah{Comment on assumption. I also think it might be ok to introduce assumption right before theory where it is needed.}
%
We denote the RL procedure applied to the cost function recovered by ATIRL as $\mathtt{RL} \circ \mathtt{ATIRL}_{\psi}(\Tfun_{t})$.

\subsection{Characterizing the Policy Induced by ATIRL} \label{sec:theory}
This section shows that it is possible to bypass the ATIRL step and learn the action transformation policy directly from data.
We show that $\psi$-regularized $\mathtt{RL} \circ \mathtt{ATIRL}_{\psi}(\Tfun_{t})$ implicitly searches for policies that have a marginal transition distribution close to the target environment's, as measured by the convex conjugate of $\psi$, which we denote as $\psi^*$.
As a practical consequence, we will then be able to devise a method for minimizing this divergence through the use of generative adversarial techniques in Section \ref{sec:garat}.
But first, we state our main theoretical claim:
\begin{restatable}{theorem}{thmgarat}
 \label{thm:garat}
$\mathtt{RL} \circ \mathtt{ATIRL}_{\psi}(\Tfun_{t})$ and $\argmin_{\p_{g}} \psi^*(\rho_{g} - \rho_{t})$ induce policies that have the same marginal transition distribution, $\rho_{g}$.
\end{restatable}

To reiterate, the agent policy $\p$ is fixed.
So the only decisions affecting the marginal transition distributions are of the action transformation policy $\p_g$.
We can now state the following proposition:

\begin{restatable}{proposition}{propunique}
\label{prop:unique}
For a given $\rho_g$ generated by a fixed policy $\p$, $\Tfun_{g}$ is the  only transition function whose marginal transition distribution is $\rho_{g}$.
\end{restatable}

Proof in Appendix \ref{app:proof_unique}.
%
% Proposition \ref{prop:unique} is justification for asserting $\Tfun_{sim}$ as unique for marginal transition distribution $\rho_{sim}$ and $\Tfun_{real}$ as unique for marginal transition distribution $\rho_{real}$.
%
We can also show that if two transition functions are equal, then the optimal policy in one will be optimal in the other.

%% Because we cannot prove a one-to-one mapping of $\p_{g}$ to $\rho_\p_{g}$, there is no point in introducing an entropy term in the space of occupancy measures.

% The introduction of the action transformer also introduces a new marginal distribution $\rho_\p_{g}(s, a, \Tilde{a}, s')$, overloaded for convenience, which we call the marginal action transformer (AT) distribution.

% \begin{align}
%     \rho_{t, at}(s, a, \Tilde{a}, s') &\defd \rho_{t,at}(s) \p(a|s) \p_{g}(\Tilde{a}|s, a) \Tfun_{g}(s'|s, \Tilde{a})  \label{eqn:at_pol_t}
% \end{align}

% We can now define the causal entropy for the action transformer $\overline{\mathcal{H}}(\rho_\p_{g})$ as:

% \begin{align}
%     \overline{\mathcal{H}}(\rho_\p_{g}) &= - \sum_{s, a, \Tilde{a}, s'} \rho_{t, at}(s, a, \Tilde{a}, s') \log\left(\frac{\rho_{t, at}(s, a, \Tilde{a}, s')}{\rho_{t, at}(s, a, s')}\right)
% \end{align}

% Given the same policy $\p$, reward function $\Rfun$, discount factor  $\D$ and starting state distribution $\rho_0$, Equation \ref{eqn:marg_return} and Proposition \ref{prop:unique} show that matching the marginal transition distributions will lead to the same expected return in both Sim and Real.
% This allows us to state the lemma below.
% that $\argmax_{\p \in \Pi} \mathbb{G}_{\p}_{sim} = \argmax_{\p \in \Pi} \mathbb{G}_{\p}_{real}$, the optimal policy in sim will be the optimal policy in real.

\begin{restatable} {proposition}
{propopt}
\label{prop:opt}
If $\Tfun_{t} = \Tfun_{g}$, then $\argmax_{\p \in \pset} \mathbb{E}_{\p, \Tfun_g} [G_0] = \argmax_{\p \in \pset} \mathbb{E}_{\p, \Tfun_{t}} [G_0]$.
\end{restatable}
Proof in Appendix \ref{app:opt}.
We now prove Theorem \ref{thm:garat}, which characterizes the policy learned by $\mathtt{RL}(\Tilde{c})$ on the cost function $\Tilde{c}$ recovered by $\mathtt{ATIRL}_{\psi}(\Tfun_{t})$.

\begin{proof}[Proof of Theorem \ref{thm:garat}]
To prove Theorem \ref{thm:garat}, we prove that $\mathtt{RL} \circ \mathtt{ATIRL}_{\psi}(\Tfun_{t})$ and $\argmin_{\p_{g}} \psi^*(\rho_{g} - \rho_{t})$ result in the same marginal transition distribution.
This proof has three parts, two of which are proving that both objectives above can be formulated as optimizing over marginal transition distributions.
The third is to show that these equivalent objectives result in the same distribution.

The output of both $\mathtt{RL} \circ \mathtt{ATIRL}_{\psi}(\Tfun_{t})$ and $\argmin_{\p_{g}} \psi^*(\rho_{g} - \rho_{t})$ are policies.
To compare the marginal distributions, we first establish a different $\overline{\mathtt{RL}}\circ \overline{\mathtt{ATIRL}}_{\psi}(\Tfun_{t})$ objective that we argue has the same marginal transition distribution as $\mathtt{RL} \circ \mathtt{ATIRL}_{\psi}(\Tfun_{t})$.
We define
\begin{equation}
    \overline{\mathtt{ATIRL}}_{\psi}(\Tfun_{t}) \defd \argmax_{c \in \mathcal{C}} -\psi(c) + \left(\min_{\rho_{g} \in \mathcal{P}_{g}} \mathbb{E}_{\rho_g} \left[c(s, a, s')\right] \right) - \mathbb{E}_{\rho_{t}} \left[c(s, a, s')\right]
    \label{eqn:bar_atirl}
\end{equation}
with the same $\psi$ and $\mathcal{C}$ as Equation \ref{eqn:ATIRL}, and similar except the internal optimization for Equation \ref{eqn:ATIRL} is over $\p_g \in \pset_g$, while it is over $\rho_g \in \mathcal{P}_g$ for Equation \ref{eqn:bar_atirl}.
% \josiah{Maybe comment on the difference between 4 and 5 so no need to look back? Could also write second two terms as expectations to save some space.}
% %
% % I'm going to comment this part out since I do the math explicitly in the appendix.
% Note that $\mathbb{E}_{\Tfun_{g}} c(s, a, s') = \sum_{s, a, s'} \rho_{g}(s, a, s') c(s, a, s')$ from equations \ref{eqn:trans_g} and \ref{eqn:rho_g} with both dependent on the action transformation policy $\p_g$.
% Therefore, Equations \ref{eqn:ATIRL} and \ref{eqn:bar_atirl} are the same except the first is optimized over $\p_{g}$ and the second over $\rho_{g}$.
% \ishan{I'm not completely certain I've been able to catch the subtlety Garrett mentioned.}
% \garrett{I don't think the equations are *quite* the same --- in (2), you have $\mathbb{E}_{P_g}$, but in the previous sentence you used $\mathbb{E}_{\pi_g}$. The difference is subtle, but should be corrected.}
% \garrett{Additionally, I found the wording of the previous sentence to be a bit confusing. Why say "as long as the agent policy $\pi$ is the same?" Haven't we stated that it's fixed here? I think you should say instead "because the agent policy $\pi$ is fixed, ..."}
We define an RL procedure $\overline{\mathtt{RL}}(\overline{c}) \defd \argmin_{\rho_{g} \in \mathcal{P}_{g}} \mathbb{E}_{\rho_g} c(s, a, s')$ that returns a marginal transition distribution $\rho_{g} \in \mathcal{P}_{g}$ which minimizes the given cost function $\overline{c}$.
$\overline{\mathtt{RL}}(\overline{c})$ will output the marginal transition distribution $\overline{\rho}_{g}$.
% \begin{align}
%     \overline{\mathtt{RL}}(\overline{c}) &\defd \argmin_{\rho_{g} \in \mathcal{P}_{g}} \sum_{s, a, s'} \rho_{g}(s, a, s') c(s, a, s') \label{eqn:bar_rl}
% \end{align}
% Equation \ref{eqn:bar_rl} will output the marginal transition distribution $\overline{\rho}_{g}$.
% \josiah{Do you mean the RHS is the marginal transition distribution?}

\begin{restatable}{lemma}{lemoverlineequi}
\label{lem:overline_equi}
$\overline{\mathtt{RL}}\circ \overline{\mathtt{ATIRL}}_{\psi}(\Tfun_{t})$ outputs a marginal transition distribution $\overline{\rho}_{g}$ which is equal to $\Tilde{\rho}_{g}$ induced by $\mathtt{RL} \circ \mathtt{ATIRL}_{\psi}(\Tfun_{t})$.
\end{restatable}
Proof in Appendix \ref{app:lemoverline_equi_proof}.
% \garrett{I have left comments in this Appendix.}
%
% Note that the policies that return the above marginal transition distributions are not necessarily the same.
% \garrett{This is a confusing statement since the "barred" procedure doesn't return a policy, but rather a marginal distribution --- you say it just a few lines up above the Lemma 4.1 statement.}
The mapping from $\pset_{g}$ to $\mathcal{P}_{g}$ is not injective, and there could be multiple policies $\p_{g}$ that lead to the same marginal transition distribution.
The above lemma is sufficient for proof of Theorem \ref{thm:garat}, however, since we focus on the effect of the policy on the transitions.
% \ishan{Try to reduce the previous 3 sentences.}
% \garrett{"sufficient" for what?}

\begin{restatable} {lemma}{lemoverlinepsiequi}
 \label{lem:overline_psi_equi}
$\overline{\mathtt{RL}}\circ \overline{\mathtt{ATIRL}}_{\psi}(\Tfun_{t}) = \argmin_{\rho_{g} \in \mathcal{P}_{g}} \psi^*(\rho_{g} - \rho_{t})$.
\end{restatable}
% Proof in Appendix \ref{app:lem_proof}. \garrett{I have left comments in this Appendix.}
The proof in Appendix \ref{app:lem_proof} relies on the optimal cost function and the optimal policy forming a saddle point, $\psi^*$ leading to a minimax objective, and these objectives being the same.

\begin{restatable}{lemma}{lempsiequivalence}
The marginal transition distribution of
$\argmin_{\p_{g}} \psi^*(\rho_{g} - \rho_{t})$ is equal to $\argmin_{\rho_{g} \in \mathcal{P}_{g}} \psi^*(\rho_{g} - \rho_{t})$.
\label{lem:psi_equi}
\end{restatable}
% \josiah{More formal to say "is equal to"}
Proof in appendix \ref{app:proof_psi_equi}.
% \garrett{You should say "Lemma" 4.1. Also, I'd recommend putting in a small section for this separately instead of saying it's similar. Especially since I'm not currently following the proof of Lemma 4.1.}
%
With these three lemmas, we have proved that  $\mathtt{RL} \circ \mathtt{ATIRL}_{\psi}(\Tfun_{t})$ and $\argmin_{\p_{g}} \psi^*(\rho_{g} - \rho_{t})$ induce policies that have the same marginal transition distribution.
\end{proof}

Theorem \ref{thm:garat} thus tells us that the objective $\argmin_{\p_{g}} \psi^*(\rho_{g} - \rho_{t})$ is equivalent to the procedure from Section \ref{sec:atirl}.
In the next section, we choose a function $\psi$ which leads to our adversarial objective.
% Now that we have an objective that bypasses the ATIRL step and can learn from the data directly, the next step is to instantiate it with an appropriate cost regularizer.

% These three lemmas conclude the proof of Theorem \ref{thm:garat}.
% The intuition here is that the $\psi$-regularized  inverse reinforcement learning seeks an action transformer that leads to a marginal transition distribution close to the real world's, as measured by the convex function $\psi^*$.
% It results in the insight that we can use the objective $\argmin_{\p_{g}} \psi^*(\rho_{g} - \rho_{real})$ to find an action transformer policy that imitates the real world transitions for the agent policy $\p$.
% \garrett{The end of this paragraph is the intuition I'm craving at the beginning of the section. Still missing the motivation (something about how we want a practical algorithm, and we think GANs might be that algorithm but they really only work for objectives of a certain type and wouldn't it be great if we could show that our problem yielded exactly such an objective etc, etc.).}

\subsection{Forming the Adversarial Objective}
\label{sec:garat}
%
% \josiah{Perhaps the final objective should be set out more prominently in the previous section with some motivation. I think right now it's strongly implied what objective you want but it could be laid out more explicitly.}
Section \ref{sec:theory} laid out the objective we want to minimize.
% This section now lays out a practical algorithm that can be used to learn an action transformer to imitate the real world dynamics as closely as possible.
To solve $\argmin_{\p_{g}} \psi^*(\rho_{g} - \rho_{t})$ we require an appropriate regularizer $\psi$.
\gail{} \cite{ho2016gail} and \gaifo{} \cite{torabi2018generative} optimize similar objectives and have shown a regularizer similar to the following to work well:
% \commentp{I can't parse the first part of the equation.  What are the conditions for activating the top or bottom line?}
\begin{align}
    \psi(c) = \begin{cases}
    \mathbb{E}_{t} [ g(c(s, a, s'))] & \text{if } c < 0 \\
    +\infty & \text{otherwise}
    \end{cases}
    \text{where }
    g(x) = \begin{cases}
          -x - log(1 - e^x) &\text{if } x < 0 \\
          + \infty  & \text{otherwise}
    \end{cases}
\end{align}
% \begin{align}
%     \psi(c) = \left\{
%     \begin{array}{cc}
%          \mathbb{E}_{real} [ g(c(s, a, s'))] & \text{if } c < 0 \\
%          +\infty & \text{otherwise}
%     \end{array}\right.
%     \text{where }
%     g(x) = \left\{
%     \begin{array}{cc}
%           -x - log(1 - e^x) &\text{if } x < 0 \\
%           + \infty  & \text{otherwise}
%     \end{array}
%     \right.
% \end{align}
It is closed, proper, convex and has a convex conjugate leading to the following minimax objective:
% \begin{align}
%     \psi^*(\rho_g - \rho_{real}) &= \max_{D\in (0, 1)^{\sset\times\aset\times\sset}} \sum_{s, a, s'} \rho_g(s, a, s')\log D(s, a, s') +  \\\nonumber
%     &\rho_{real}(s, a, s') \log(1 - D(s, a, s')) &&
% \end{align}
% \garrett{I actually think you could drop the above equation and move right to (9) below.}
% where $D: \sset \times \aset \times \sset \longmapsto (0, 1)$ is a discriminative classifier.
% The action transformation policy can then be solved as:
\begin{align}
    \min_{\p_g \in \pset_g} \psi^*(\rho_g - \rho_{t}) = \min_{\p_g \in \pset_g} \max_{D} \mathbb{E}_{\Tfun_g}[\log (D(s, a, s'))] + \mathbb{E}_{\Tfun_{t}}[ \log(1 - D(s, a, s'))]
\end{align}
where the reward for the action transformer policy  $\p_g$ is $-[\log(D(s, a, s'))]$, and $D: \sset \times \aset \times \sset \longmapsto (0, 1)$ is a discriminative classifier.
These properties have been shown in previous works \cite{ho2016gail, torabi2018generative}.
Algorithm \ref{alg} lays out the steps for learning the action transformer using the above procedure, which we call generative adversarial reinforced action transformation (\garat{}).

\section{Related Work} \label{sec:related}
% This section points out related work in sim-to-real transfer and imitation learning and compares \garat{} to these works.
%
% In this section, we review existing literature on sim-to-real and imitation learning. We begin by discussing the variety of sim-to-real methods, then introduce work closely related to \garat{}, and finally, detail related work in the IfO literature. \sid{This paragraph seems unnecessary to me.}
% \ishan{Garrett suggested easing into the related work. Maybe we can just turn it into a sentence.}\sid{yeah, a sentence is fine}

While our work lies in the space of transfer learning with dynamics mismatch, the eventual goal of this research is to enable effective sim-to-real transfer.
In this section, we discuss the variety of sim-to-real methods, work more closely related to \garat{}, and some related methods in the IfO literature.
Sim-to-real transfer can be improved by making the agent's policy more robust to variations in the environment or by making the simulator more accurate w.r.t.\ the real world.
The first approach, which we call policy robustness methods, encompasses algorithms that train a robust policy that performs well on a range of environments \cite{jakobi1997evolutionary, peng2018sim, peng2020learning, pinto2017robust, epopt, sadeghi2016cad2rl, tobin2017domain, tobin2018domain}.
Robust adversarial reinforcement learning (\rarl{}) \cite{pinto2017robust} is such an algorithm that learns a policy robust to adversarial perturbations \cite{szegedy2013intriguing}.
While primarily focused on training with a modifiable simulator, a version of \rarl{} treats the simulator as a black-box by adding the adversarial perturbation directly to the protagonist's action.
Additive noise envelope (\textsc{ane}) \cite{action_envelope_noise} is another black-box robustness method which adds an envelope of Gaussian noise to the agent's action during training.

The second approach, known as domain adaption or system identification, grounds the simulator using real world data to make its transitions more realistic.
Since hand engineering accurate simulators \cite{minitaur, xie2019learning} can be expensive and time consuming, real world data can be used to adapt low-fidelity simulators to the task at hand.
Most simulator adaptation methods \cite{allevato2019tunenet, chebotar2019closing,farchy2013gsl, hwangbo2019learning} rely on access to a parameterized simulator.

\garat{}, on the other hand, does not require a modifiable simulator and relies on an action transformation policy applied in the source environment to bring its transitions closer to the target environment.
\gat{}\cite{hanna2017grounded} learns an action transformation function similar to \garat{}.
It was shown to have successfully learned and transferred one of the fastest known walk policies on the humanoid robot, Nao.
% be successful at transferring a humanoid robot walk policy from simulator to the real world and produced one of the fastest known walks on the Nao \haresh{should we provide a link to the website ?}.
%
% However, \garat{} utilizes reinforcement learning and an objective that focuses on minimizing the transition mismatch over entire trajectories, while \gat{} minimizes the per-step transition error which might lead to compounding errors over the course of trajectories. Also, as we show in the experiments, \garat{} can transfer policies from simulator to the real world using minimal real world interactions, a crucial requirement in the robotics domain.

\garat{} draws from recent generative adversarial approaches to imitation learning (\gail{} \cite{ho2016gail}) and IfO (\gaifo{} \cite{torabi2018generative}).
%
%  and  show that a marginal distribution matching approach has the same effect as the combined IRL followed by RL procedure, and use a generative adversarial algorithm to do so effectively.
%
\textsc{airl}\cite{fu2018learning}, \textsc{fairl}\cite{ghasemipour2019divergence}, and \textsc{wail}\cite{xiao2019wasserstein} are related approaches which use different divergence metrics to reduce the marginal distribution mismatch.
\garat{} can be adapted to use any of these metrics, as we show in the appendix.

One of the insights of this paper is that grounding the simulator using action transformation can be seen as a form of IfO.
\textsc{bco} \cite{torabi2018behavioral} is an IfO technique that utilizes behavioral cloning.
% It is interesting to note here the similarity between \textsc{bco} and \gat{}, which both ignore the sequential nature of their problems and utilize an inverse dynamics model $\Tfun^{-1}(a|s, s')$ to predict the optimal action to take in a state.
% \josiah{Do we want to point out that GAT is similar to BCO? I think we talked about this and decided similar but not identical.}
% \ishan{I tried. But I'm not sure how much value there is to saying they are similar without saying how or giving any additional explanation. We removed the earlier part where we explained the similarity.}
% uses interactions with the environment to learn an inverse dynamics model $(\Tfun^{-1}(s, s') \longmapsto \Delta(\aset))$ to predict the actions most likely to result in the given state transitions, and using this model to generate labels for behavioral cloning.
%
% \textsc{bco} then uses the predictions of this model to label the expert's state-to-state transitions.
% %
% A standard behavioral cloning method can then be used to learn an agent policy thereafter.
%
% \gaifo{} \cite{torabi2018generative} is an adversarial algorithm (similar to \gail{}) to solve the \emph{IfO} problem more effectively than \textsc{bco}.
%
\textsc{i2l} \cite{Gangwani2020State-only} is an \emph{IfO} algorithm that aims to learn in the presence of transition dynamics mismatch in the expert and agent's domains, but requires millions of real world interactions to be competent.
%
% While this setting seems similar to sim-to-real transfer, it is important to note that the agent does not \emph{transfer} a policy from one domain to another.
% It learns a policy in the target environment such that the marginal state-transition distribution is similar to an expert in the source environment.
% %
% The result is a policy that requires millions of target environment (the real world in our case) interactions to be competent, whereas our focus is on reducing the need for target environment interactions.

\section{Experiments} \label{sec:exp}

In this section, we conduct experiments to verify our hypothesis that \garat{} leads to improved transfer in the presence of dynamics mismatch compared to previous methods.
We also show that it leads to better source environment grounding compared to the previous action transformation approach, \gat{}.
% \josiah{This intro just says GAT but we consider other action transformation baselines too.}
% \ishan{THis entire paragraph is missing the experiments on sim-to-real transfer. It's just talking about the grounding. We should make it more representative of the entire section.}
% In this section, we conduct experiments to verify our hypothesis that \garat{} can learn a more effective action transformation policy compared to \gat{} by showing that it does indeed reduce the difference in the dynamics between two environments and this improved grounding translates to better transfer of policies from one environment to the other. 

% \garrett{In the introduction, we did a nice job of framing a hypothesis and talking about how we were going to run experiments to validate that hypothesis. That should be repeated here at the start of the experiments section.} \haresh{resolved}
% We performed several experiments to investigate the efficiency of \garat{} for sim-to-real transfer.
% The experimental section validates the hypothesis from the rest of the paper that the adversarial imitation learning algorithm we derived is more efficient at sim-to-real transfer than previous grounded sim-to-real techniques.
%
We validate \garat{} for transfer by transferring the agent policy between Open AI Gym \cite{brockman2016openai} simulated environments with different transition dynamics.
We highlight the Minitaur domain (Figure \ref{fig:minitaur}) as a particularly useful test since there exist two simulators, one of which has been carefully engineered for high fidelity to the real robot \cite{minitaur}.
% To validate our algorithm for sim-to-real transfer, in this paper we focus on transferring an agent policy between
% As a proxy for sim-to-real transfer, we transfer agent policies between
% two simulated environments with different transition dynamics.
For other environments, the target environment is the source environment modified in different ways such that a policy trained in the source environment does not transfer well to the target environment.
Details of these modifications are provided in Appendix \ref{environment_modifications}. Apart from a thorough evaluation across multiple different domains, this setup also allows us to compare \garat{} and other algorithms against a policy trained directly in the target environment with millions of interactions, which is otherwise prohibitively expensive on a real robot.
This setup also allows us to perform a thorough evaluation of sim-to-real algorithms across multiple different domains.
%
% Throughout this section, we refer to the target environment as the ``real'' environment and the source environment as the simulator.
%
% In the experimental section, we want to verify that Algorithm \ref{alg} does indeed lead to improved transfer from Sim-to-Real.
% To verify better sim-to-real transfer, we seek to answer the following questions:
We focus here on answering the following questions :
\begin{enumerate}
    \itemsep0em 
    \item How well does \textsc{garat} ground the source environment with respect to the target environment?
    \item Does \textsc{garat} lead to improved transfer with in the presence of dynamics mismatch, compared to other related methods?
\end{enumerate}

\subsection{Source Environment Grounding} \label{sec:sim_ground}

\begin{figure}[!tb]
    \centering
    \begin{subfigure}[b]{0.49\textwidth}
        \centering
        \includegraphics[width=\textwidth]{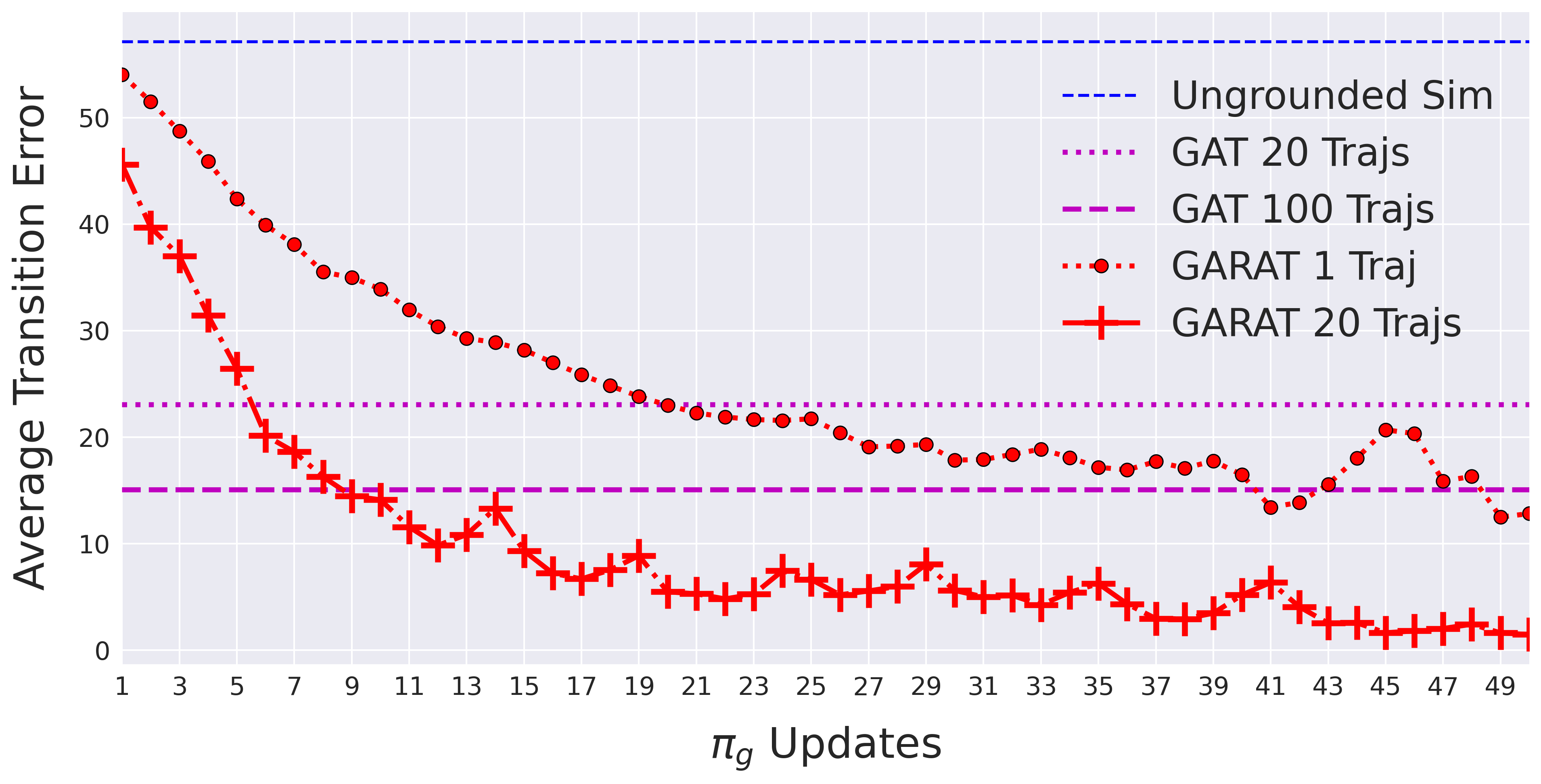}
        \caption{$\mathtt{L2}$ norm of per step transition errors (lower is better) between different source environments and the target environment, shown over number of action transformation policy updates for \garat{}.}
        \label{fig:persteptransitionerrors}
    \end{subfigure}
    \hfill
    \begin{subfigure}[b]{0.49\textwidth}
        \centering
        \includegraphics[width=\textwidth]{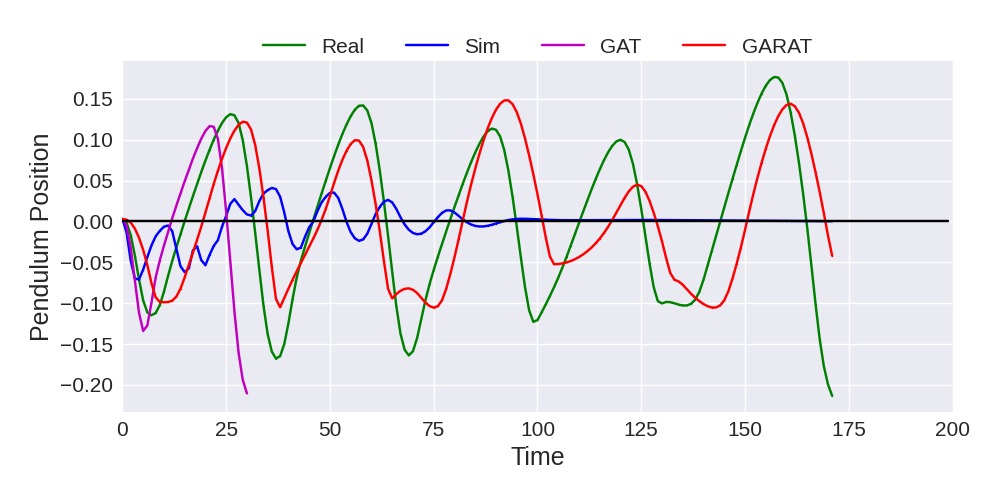}
        % \caption{Value of pendulum angle over time for policies trained in different environments and deployed from the same start state in the real world. Policy trained in \textsc{garat}-grounded simulator produces a response similar to $\pi_{real}$.}
        \caption{Example trajectories of the same agent policy deployed in different environments, plotted using the pendulum angle across time. Response of \garat{} grounded source environment is the most like target environment.}
        \label{fig:trajs}
    \end{subfigure}
    \caption{Evaluation of source environment grounding with \garat{} in \textit{InvertedPendulum} domain}
    \label{fig:trans_error}
\end{figure}

%%%%%%%%%%%%%%%%%%%%%%%%%%%%%%%%%%%%%%%%%%%%%%%%%%
% Ishan: Attempting to compress this subsection
In Figure \ref{fig:trans_error}, we evaluate how well \textsc{garat} grounds the source environment to the target environment both quantitatively and qualitatively.
This evaluation is in the \textit{InvertedPendulum} domain, where the target environment has a heavier pendulum than the source; implementation details are in Appendix \ref{environment_modifications}.
In Figure \ref{fig:persteptransitionerrors}, we plot the average error in transitions in source environments grounded with \garat{} and \gat{} with different amounts of target environment data, collected by deploying $\pi$ in the target environment.
In Figure \ref{fig:trajs} we deploy the same policy $\p$ from the same start state in the different environments (source, target, and grounded source).
From both these figures it is evident that \garat{} leads to a grounded source environment with lower error on average, and responses qualitatively closer to the target environment compared to \gat{}.
Details of how we obtained these plots are in Appendix \ref{app:sim_ground}.

\subsection{Transfer Experiments} \label{sec:exp_transfer}
\begin{wrapfigure}{R}{0.4\textwidth}
  \begin{center}
    \includegraphics[width=0.4\textwidth]{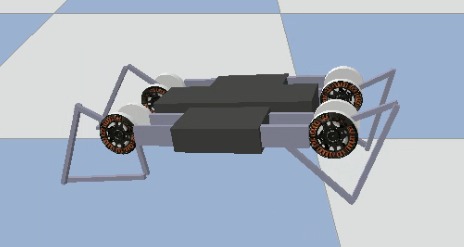}
  \end{center}
  \caption{The Minitaur Domain}
  \label{fig:minitaur}
\end{wrapfigure}

We now validate the effectiveness of \textsc{garat} at transferring a policy from source environment to target environment.
For various MuJoCo \cite{todorov2012mujoco} environments, we pretrain the agent policy $\p$ in the ungrounded source environment, collect target environment data with $\pi$, use \textsc{garat} to ground the source environment, re-train the agent policy until convergence in these grounded source environments, and then evaluate mean return across 50 episodes for the updated agent policy in the target environment.

The agent policy $\p$ and action transformation policy $\pi_g$ are trained with \textsc{trpo} \cite{TRPO_algo} and \textsc{ppo} \cite{PPO_algo} respectively.
The specific hyperparameters used are provided in Appendix \ref{expt_details}.
We use the implementations of \textsc{trpo} and \textsc{ppo} provided in the stable-baselines library \cite{stable-baselines}.
For every $\p_g$ update, we update the \garat{} discriminator $D_{\phi}$ once as well.
Results here use the losses detailed in Algorithm \ref{alg}.
However, we find that \garat{} is just as effective with other divergence measures \cite{fu2018learning, ghasemipour2019divergence, xiao2019wasserstein} (Appendix \ref{expt_details}).

% \ishan{This seems like a repetition of the last paragraph before section 6.1. But even this one can be reduced significantly. We can just say used these methods. Which hyperparameter we used and how we found them can be explained in appendix.}
\garat{} is compared to \gat{} \cite{hanna2017grounded}, \textsc{rarl} \cite{pinto2017robust} adapted for a black-box simulator, and action-noise-envelope (\textsc{ane}) \cite{action_envelope_noise}.
$\pi_{t}$ and $\pi_{s}$ denote policies trained in the target environment and source environment respectively until convergence.
We use the best performing hyperparameters for these methods, specified in Appendix \ref{expt_details}.

\begin{figure}[!bt]
    \centering
    \includegraphics[width=\textwidth]{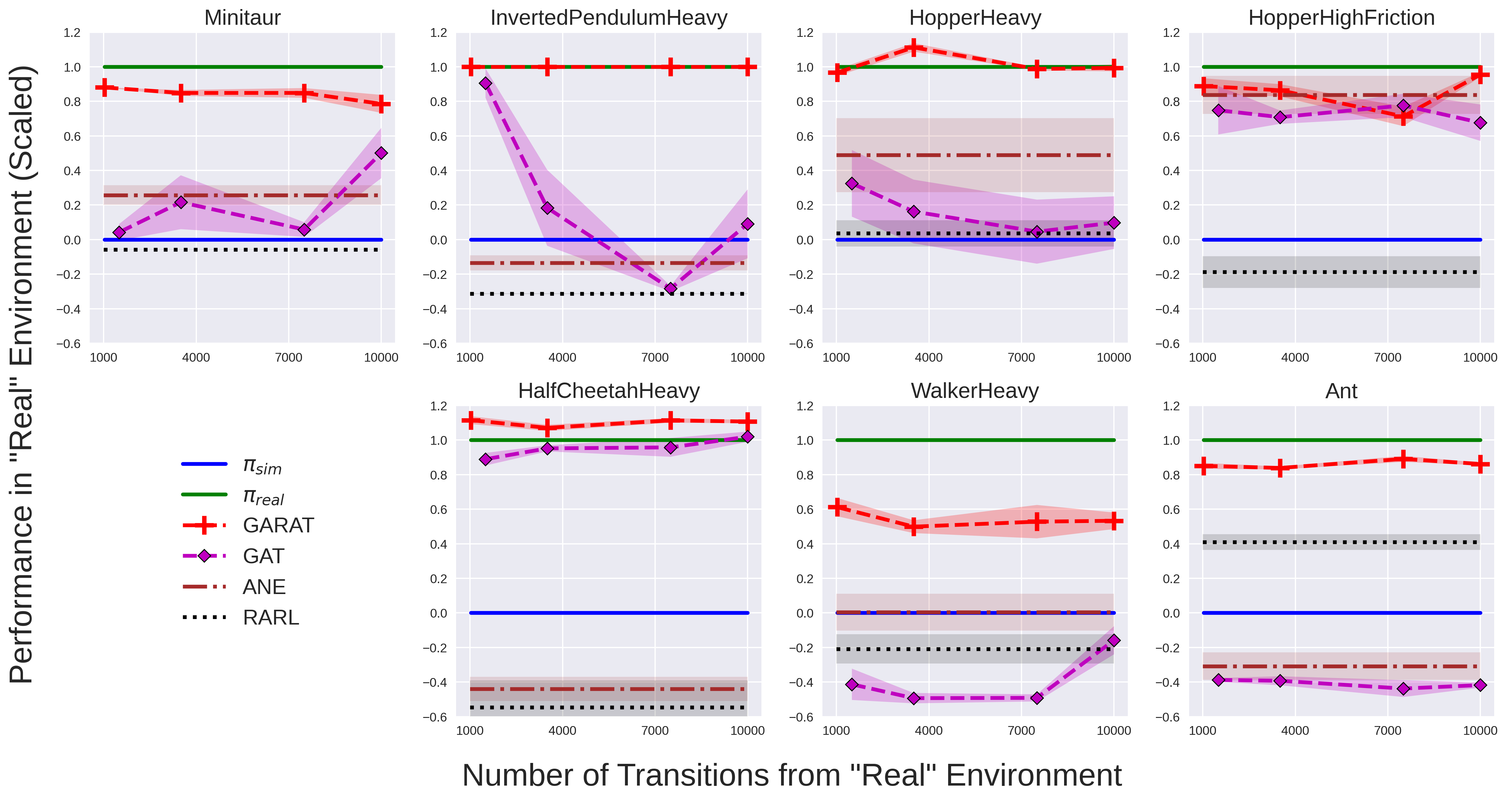}
    \caption{Performance of different techniques evaluated in target environment. Environment return on the $y$-axis is scaled such that $\pi_{t}$ achieves 1 and $\pi_{s}$ achieves 0.
    }
    \label{fig:results}
\end{figure}

Figure \ref{fig:results} shows that, in most of the domains, \garat{} with just a few thousand transitions from the target environment facilitates transfer of policies that perform on par with policies trained directly in the target environment using 1 million transitions.
\garat{} also consistently performs better than previous methods on all domains, except \emph{HopperHighFriction}, where most of the methods perform well.
The shaded envelope denotes the standard error across 5 experiments with different random seeds for all the methods. 
Apart from the MuJoCo simulator, we also show successful transfer in the PyBullet simulator \cite{coumans2016pybullet} using the \emph{Ant} domain.
Here the target environment has gravity twice that of the source environment, resulting in purely source environment-trained policies collapsing ineffectually in the target environment.
In this relatively high dimensional domain, as well as in \emph{Walker}, we see \garat{} still transfers a competent policy while the related methods fail.
% For the Ant domain we validate transfer using the PyBullet \cite{coumans2016pybullet} simulator. 
% %
% An initial policy learned in a simulator with half the gravity of ``real'' collapses on the ground after a few steps in ``real'' due to the higher gravity. However, after retraining in an environment grounded with \textsc{garat}, the Ant agent is able to successfully walk in the "real" world.

In the Minitaur domain \cite{minitaur} we use the high fidelity simulator as our target environment.
Here as well, a policy trained in the source environment does not directly transfer well to the target environment \cite{policy-strategy-optimization}.
We see in this realistic setting that \garat{} learns a policy that obtains more than 80\% of the optimal target environment performance with just $1000$ target environment transitions while the next best baseline (\gat{}) obtains at most 50\%, requiring ten times more target environment data.

\section{Conclusion}

In this paper, we have shown that grounded action transformation, a particular kind of grounded transfer technique, can be seen as a form of imitation from observation.
We use this insight to develop \garat{}, an adversarial imitation from observation algorithm for grounded transfer.
% , which we show leads to improved transfer of an agent policy from the grounded simulator to the real world.
%
% This hypothesis is validated in our experiments, where \garat{} leads to better sim-to-real transfer as compared to \gat{} as well as other black-box transfer techniques.
%
We hypothesized that such an algorithm would lead to improved grounding of the source environment as well as better transfer compared to related techniques.
This hypothesis is validated in Section \ref{sec:exp} where we show that \garat{} leads to better grounding of the source environment as compared to \gat{}, and improved transfer to the target environment on various mismatched environment transfers, including the realistic Minitaur domain.

\section*{Acknowledgements and Disclosure of Funding}
This work has taken place in the Learning Agents Research
Group (LARG) at the Artificial Intelligence Laboratory, The University of Texas at Austin.
LARG research is supported in part by grants from the National Science Foundation (CPS-1739964, IIS-1724157, NRI-1925082), the Office of Naval Research (N00014-18-2243), Future of
Life Institute (RFP2-000), Army Research Office (W911NF-19-2-0333), DARPA, Lockheed Martin, General Motors, and Bosch.
The views and conclusions contained in this document are those of the authors alone.
Peter Stone serves as the Executive Director of Sony AI America and receives financial compensation for this work.  
The terms of this arrangement have been reviewed and approved by the University of Texas at Austin in accordance with its policy on objectivity in research.

\section*{Broader Impact}
Reinforcement learning \cite{sutton2018reinforcement} is being considered as an effective tool to train autonomous agents in various important domains like robotics, medicine, etc.
A major hurdle to deploying learning agents in these environments is the massive exploration and data requirements \cite{hanna2019data} to ensure that these agents learn effective policies.
Real world interactions and exploration in these situations could be extremely expensive (wear and tear on expensive robots), or dangerous (treating a patient in the medical domain).

Sim-to-real transfer aims to address this hurdle and enables agents to be trained mostly in simulation and then transferred to the real world based on very few interactions.
Reducing the requirement for real world data for autonomous agents might open up the viability for autonomous agents in other fields as well.

Improved sim-to-real transfer will also reduce the pressure for high fidelity simulators, which require significant engineering effort \citep{chebotar2019closing, minitaur}.
Simulators are also developed with a task in mind, and are generally not reliable outside their specifications.
Sim-to-real transfer might enable simulators that learn to adapt to the task that needs to be performed, a potential direction for future research.

Sim-to-real research needs to be handled carefully, however.
Grounded simulators might lead to a false sense of confidence in a policy trained in such a simulator.
However, a simulator grounded with real world data will still perform poorly in situations outside the data distribution.
As has been noted in the broader field of machine learning \cite{amodei2016concrete}, out of training distribution situations might lead to unexpected consequences.
Simulator grounding must be done carefully in order to guarantee that the grounding is applied over all relevant parts of the environment.

Improved sim-to-real transfer could increase reliance on compute and reduce incentives for sample efficient methods.
The field should be careful in not abandoning this thread of research as the increasing cost and impact of computation used by machine learning becomes more apparent \cite{amodei2018ai}.

\bibliographystyle{plainnat}
\bibliography{refs}

\begin{thebibliography}{53}
\providecommand{\natexlab}[1]{#1}
\providecommand{\url}[1]{\texttt{#1}}
\expandafter\ifx\csname urlstyle\endcsname\relax
  \providecommand{\doi}[1]{doi: #1}\else
  \providecommand{\doi}{doi: \begingroup \urlstyle{rm}\Url}\fi

\bibitem[Allevato et~al.(2019)Allevato, Short, Pryor, and
  Thomaz]{allevato2019tunenet}
Adam Allevato, Elaine~Schaertl Short, Mitch Pryor, and Andrea~L Thomaz.
\newblock Tunenet: One-shot residual tuning for system identification and
  sim-to-real robot task transfer.
\newblock In \emph{Conference on Robot Learning (CoRL)}, 2019.

\bibitem[Amodei and Hernandez(2018)]{amodei2018ai}
Dario Amodei and Danny Hernandez.
\newblock {AI} and compute.
\newblock \emph{openai.com}, May 2018.
\newblock URL \url{https://openai.com/blog/ai-and-compute/}.

\bibitem[Amodei et~al.(2016)Amodei, Olah, Steinhardt, Christiano, Schulman, and
  Man{\'e}]{amodei2016concrete}
Dario Amodei, Chris Olah, Jacob Steinhardt, Paul Christiano, John Schulman, and
  Dan Man{\'e}.
\newblock Concrete problems in {AI} safety.
\newblock \emph{arXiv preprint arXiv:1606.06565}, 2016.

\bibitem[Bain and Sammut(1995)]{bain1995framework}
Michael Bain and Claude Sammut.
\newblock A framework for behavioural cloning.
\newblock In \emph{Machine Intelligence 15}, pages 103--129, 1995.

\bibitem[Bakker and Kuniyoshi(1996)]{bakker1996robot}
Paul Bakker and Yasuo Kuniyoshi.
\newblock Robot see, robot do: An overview of robot imitation.
\newblock In \emph{AISB96 Workshop on Learning in Robots and Animals}, pages
  3--11, 1996.

\bibitem[Bousmalis et~al.(2017)Bousmalis, Irpan, Wohlhart, Bai, Kelcey,
  Kalakrishnan, Downs, Ibarz, Pastor, Konolige, Levine, and Vanhoucke]{R24}
Konstantinos Bousmalis, Alex Irpan, Paul Wohlhart, Yunfei Bai, Matthew Kelcey,
  Mrinal Kalakrishnan, Laura Downs, Julian Ibarz, Peter Pastor, Kurt Konolige,
  Sergey Levine, and Vincent Vanhoucke.
\newblock Using simulation and domain adaptation to improve efficiency of deep
  robotic grasping.
\newblock \emph{CoRR}, abs/1709.07857, 2017.
\newblock URL \url{http://arxiv.org/abs/1709.07857}.

\bibitem[Brockman et~al.(2016)Brockman, Cheung, Pettersson, Schneider,
  Schulman, Tang, and Zaremba]{brockman2016openai}
Greg Brockman, Vicki Cheung, Ludwig Pettersson, Jonas Schneider, John Schulman,
  Jie Tang, and Wojciech Zaremba.
\newblock Openai gym.
\newblock \emph{arXiv preprint arXiv:1606.01540}, 2016.

\bibitem[Chebotar et~al.(2019)Chebotar, Handa, Makoviychuk, Macklin, Issac,
  Ratliff, and Fox]{chebotar2019closing}
Yevgen Chebotar, Ankur Handa, Viktor Makoviychuk, Miles Macklin, Jan Issac,
  Nathan Ratliff, and Dieter Fox.
\newblock Closing the sim-to-real loop: Adapting simulation randomization with
  real world experience.
\newblock In \emph{2019 International Conference on Robotics and Automation
  (ICRA)}, pages 8973--8979. IEEE, 2019.

\bibitem[Coumans and Bai(2016)]{coumans2016pybullet}
Erwin Coumans and Yunfei Bai.
\newblock Pybullet, a python module for physics simulation for games, robotics
  and machine learning.
\newblock \emph{GitHub repository}, 2016.

\bibitem[Farchy et~al.(2013)Farchy, Barrett, MacAlpine, and
  Stone]{farchy2013gsl}
Alon Farchy, Samuel Barrett, Patrick MacAlpine, and Peter Stone.
\newblock Humanoid robots learning to walk faster: From the real world to
  simulation and back.
\newblock In \emph{Proc. of 12th Int. Conf. on Autonomous Agents and Multiagent
  Systems (AAMAS)}, May 2013.

\bibitem[Fu et~al.(2018)Fu, Luo, and Levine]{fu2018learning}
Justin Fu, Katie Luo, and Sergey Levine.
\newblock Learning robust rewards with adverserial inverse reinforcement
  learning.
\newblock In \emph{International Conference on Learning Representations}, 2018.
\newblock URL \url{https://openreview.net/forum?id=rkHywl-A-}.

\bibitem[Gangwani and Peng(2020)]{Gangwani2020State-only}
Tanmay Gangwani and Jian Peng.
\newblock State-only imitation with transition dynamics mismatch.
\newblock In \emph{International Conference on Learning Representations}, 2020.
\newblock URL \url{https://openreview.net/forum?id=HJgLLyrYwB}.

\bibitem[Ghasemipour et~al.(2019)Ghasemipour, Zemel, and
  Gu]{ghasemipour2019divergence}
Seyed Kamyar~Seyed Ghasemipour, Richard Zemel, and Shixiang Gu.
\newblock A divergence minimization perspective on imitation learning methods,
  2019.

\bibitem[Goodfellow et~al.(2014)Goodfellow, Pouget-Abadie, Mirza, Xu,
  Warde-Farley, Ozair, Courville, and Bengio]{goodfellow2014generative}
Ian Goodfellow, Jean Pouget-Abadie, Mehdi Mirza, Bing Xu, David Warde-Farley,
  Sherjil Ozair, Aaron Courville, and Yoshua Bengio.
\newblock Generative adversarial nets.
\newblock In \emph{Advances in neural information processing systems}, pages
  2672--2680, 2014.

\bibitem[Hanna and Stone(2017)]{hanna2017grounded}
Josiah~P Hanna and Peter Stone.
\newblock Grounded action transformation for robot learning in simulation.
\newblock In \emph{Thirty-First AAAI Conference on Artificial Intelligence},
  2017.

\bibitem[Hanna(2019)]{hanna2019data}
Josiah~Paul Hanna.
\newblock \emph{Data efficient reinforcement learning with off-policy and
  simulated data}.
\newblock PhD thesis, University of Texas at Austin, 2019.

\bibitem[Hill et~al.(2018)Hill, Raffin, Ernestus, Gleave, Kanervisto, Traore,
  Dhariwal, Hesse, Klimov, Nichol, Plappert, Radford, Schulman, Sidor, and
  Wu]{stable-baselines}
Ashley Hill, Antonin Raffin, Maximilian Ernestus, Adam Gleave, Anssi
  Kanervisto, Rene Traore, Prafulla Dhariwal, Christopher Hesse, Oleg Klimov,
  Alex Nichol, Matthias Plappert, Alec Radford, John Schulman, Szymon Sidor,
  and Yuhuai Wu.
\newblock Stable baselines.
\newblock \url{https://github.com/hill-a/stable-baselines}, 2018.

\bibitem[Ho and Ermon(2016)]{ho2016gail}
Jonathan Ho and Stefano Ermon.
\newblock Generative adversarial imitation learning.
\newblock In D.~D. Lee, M.~Sugiyama, U.~V. Luxburg, I.~Guyon, and R.~Garnett,
  editors, \emph{Advances in Neural Information Processing Systems 29}, pages
  4565--4573. Curran Associates, Inc., 2016.
\newblock URL
  \url{http://papers.nips.cc/paper/6391-generative-adversarial-imitation-learning.pdf}.

\bibitem[Hwangbo et~al.(2019)Hwangbo, Lee, Dosovitskiy, Bellicoso, Tsounis,
  Koltun, and Hutter]{hwangbo2019learning}
Jemin Hwangbo, Joonho Lee, Alexey Dosovitskiy, Dario Bellicoso, Vassilios
  Tsounis, Vladlen Koltun, and Marco Hutter.
\newblock Learning agile and dynamic motor skills for legged robots.
\newblock \emph{Science Robotics}, 4\penalty0 (26):\penalty0 eaau5872, 2019.

\bibitem[Jakobi(1997)]{jakobi1997evolutionary}
Nick Jakobi.
\newblock Evolutionary robotics and the radical envelope-of-noise hypothesis.
\newblock \emph{Adaptive behavior}, 6\penalty0 (2):\penalty0 325--368, 1997.

\bibitem[Jakobi et~al.(1995)Jakobi, Husbands, and
  Harvey]{action_envelope_noise}
Nick Jakobi, Phil Husbands, and Inman Harvey.
\newblock Noise and the reality gap: The use of simulation in evolutionary
  robotics.
\newblock In Federico Mor{\'a}n, Alvaro Moreno, Juan~Juli{\'a}n Merelo, and
  Pablo Chac{\'o}n, editors, \emph{Advances in Artificial Life}, pages
  704--720, Berlin, Heidelberg, 1995. Springer Berlin Heidelberg.
\newblock ISBN 978-3-540-49286-3.

\bibitem[James et~al.(2017)James, Davison, and Johns]{R21}
Stephen James, Andrew~J. Davison, and Edward Johns.
\newblock Transferring end-to-end visuomotor control from simulation to real
  world for a multi-stage task.
\newblock \emph{CoRR}, abs/1707.02267, 2017.
\newblock URL \url{http://arxiv.org/abs/1707.02267}.

\bibitem[James et~al.(2018{\natexlab{a}})James, Bloesch, and Davison]{R26}
Stephen James, Michael Bloesch, and Andrew~J. Davison.
\newblock Task-embedded control networks for few-shot imitation learning.
\newblock \emph{CoRR}, abs/1810.03237, 2018{\natexlab{a}}.
\newblock URL \url{http://arxiv.org/abs/1810.03237}.

\bibitem[James et~al.(2018{\natexlab{b}})James, Wohlhart, Kalakrishnan,
  Kalashnikov, Irpan, Ibarz, Levine, Hadsell, and Bousmalis]{R25}
Stephen James, Paul Wohlhart, Mrinal Kalakrishnan, Dmitry Kalashnikov, Alex
  Irpan, Julian Ibarz, Sergey Levine, Raia Hadsell, and Konstantinos Bousmalis.
\newblock Sim-to-real via sim-to-sim: Data-efficient robotic grasping via
  randomized-to-canonical adaptation networks.
\newblock \emph{CoRR}, abs/1812.07252, 2018{\natexlab{b}}.
\newblock URL \url{http://arxiv.org/abs/1812.07252}.

\bibitem[Liu et~al.(2018)Liu, Gupta, Abbeel, and Levine]{liu2018imitation}
YuXuan Liu, Abhishek Gupta, Pieter Abbeel, and Sergey Levine.
\newblock Imitation from observation: Learning to imitate behaviors from raw
  video via context translation.
\newblock In \emph{2018 IEEE International Conference on Robotics and
  Automation (ICRA)}, pages 1118--1125. IEEE, 2018.

\bibitem[Matas et~al.(2018)Matas, James, and Davison]{R23}
Jan Matas, Stephen James, and Andrew~J. Davison.
\newblock Sim-to-real reinforcement learning for deformable object
  manipulation.
\newblock \emph{CoRR}, abs/1806.07851, 2018.
\newblock URL \url{http://arxiv.org/abs/1806.07851}.

\bibitem[Mescheder et~al.(2018)Mescheder, Geiger, and Nowozin]{GAN-gp}
Lars Mescheder, Andreas Geiger, and Sebastian Nowozin.
\newblock Which training methods for {GAN}s do actually converge?
\newblock In Jennifer Dy and Andreas Krause, editors, \emph{Proceedings of the
  35th International Conference on Machine Learning}, volume~80 of
  \emph{Proceedings of Machine Learning Research}, pages 3481--3490,
  Stockholmsmässan, Stockholm Sweden, 10--15 Jul 2018. PMLR.
\newblock URL \url{http://proceedings.mlr.press/v80/mescheder18a.html}.

\bibitem[Ng et~al.(2000)Ng, Russell, et~al.]{ng2000algorithms}
Andrew~Y Ng, Stuart~J Russell, et~al.
\newblock Algorithms for inverse reinforcement learning.
\newblock In \emph{Icml}, volume~1, page 663–670, 2000.

\bibitem[OpenAI et~al.(2019)OpenAI, Akkaya, Andrychowicz, Chociej, Litwin,
  McGrew, Petron, Paino, Plappert, Powell, Ribas, Schneider, Tezak, Tworek,
  Welinder, Weng, Yuan, Zaremba, and Zhang]{openai2019solving}
OpenAI, Ilge Akkaya, Marcin Andrychowicz, Maciek Chociej, Mateusz Litwin, Bob
  McGrew, Arthur Petron, Alex Paino, Matthias Plappert, Glenn Powell, Raphael
  Ribas, Jonas Schneider, Nikolas Tezak, Jerry Tworek, Peter Welinder, Lilian
  Weng, Qiming Yuan, Wojciech Zaremba, and Lei Zhang.
\newblock Solving rubik's cube with a robot hand, 2019.

\bibitem[Pavse et~al.(2019)Pavse, Torabi, Hanna, Warnell, and
  Stone]{pavse2019ridm}
Brahma~S Pavse, Faraz Torabi, Josiah~P Hanna, Garrett Warnell, and Peter Stone.
\newblock Ridm: Reinforced inverse dynamics modeling for learning from a single
  observed demonstration.
\newblock \emph{arXiv preprint arXiv:1906.07372}, 2019.

\bibitem[Peng et~al.(2018)Peng, Andrychowicz, Zaremba, and Abbeel]{peng2018sim}
Xue~Bin Peng, Marcin Andrychowicz, Wojciech Zaremba, and Pieter Abbeel.
\newblock Sim-to-real transfer of robotic control with dynamics randomization.
\newblock In \emph{2018 IEEE international conference on robotics and
  automation (ICRA)}, pages 1--8. IEEE, 2018.

\bibitem[Peng et~al.(2020)Peng, Coumans, Zhang, Lee, Tan, and
  Levine]{peng2020learning}
Xue~Bin Peng, Erwin Coumans, Tingnan Zhang, Tsang-Wei Lee, Jie Tan, and Sergey
  Levine.
\newblock Learning agile robotic locomotion skills by imitating animals.
\newblock \emph{arXiv preprint arXiv:2004.00784}, 2020.

\bibitem[Pinto et~al.(2017)Pinto, Davidson, Sukthankar, and
  Gupta]{pinto2017robust}
Lerrel Pinto, James Davidson, Rahul Sukthankar, and Abhinav Gupta.
\newblock Robust adversarial reinforcement learning.
\newblock In \emph{Proceedings of the 34th International Conference on Machine
  Learning-Volume 70}, pages 2817--2826. JMLR. org, 2017.

\bibitem[Puterman(1990)]{puterman1990markov}
Martin~L Puterman.
\newblock Markov decision processes.
\newblock \emph{Handbooks in operations research and management science},
  2:\penalty0 331--434, 1990.

\bibitem[Rajeswaran et~al.(2016)Rajeswaran, Ghotra, Levine, and
  Ravindran]{epopt}
Aravind Rajeswaran, Sarvjeet Ghotra, Sergey Levine, and Balaraman Ravindran.
\newblock Epopt: Learning robust neural network policies using model ensembles.
\newblock \emph{CoRR}, abs/1610.01283, 2016.
\newblock URL \url{http://arxiv.org/abs/1610.01283}.

\bibitem[Ross et~al.(2011)Ross, Gordon, and Bagnell]{ross2011reduction}
St{\'e}phane Ross, Geoffrey Gordon, and Drew Bagnell.
\newblock A reduction of imitation learning and structured prediction to
  no-regret online learning.
\newblock In \emph{Proceedings of the fourteenth international conference on
  artificial intelligence and statistics}, pages 627--635, 2011.

\bibitem[Sadeghi and Levine(2016)]{sadeghi2016cad2rl}
Fereshteh Sadeghi and Sergey Levine.
\newblock Cad2rl: Real single-image flight without a single real image.
\newblock \emph{arXiv preprint arXiv:1611.04201}, 2016.

\bibitem[Sadeghi et~al.(2017)Sadeghi, Toshev, Jang, and Levine]{R22}
Fereshteh Sadeghi, Alexander Toshev, Eric Jang, and Sergey Levine.
\newblock Sim2real view invariant visual servoing by recurrent control.
\newblock \emph{CoRR}, abs/1712.07642, 2017.
\newblock URL \url{http://arxiv.org/abs/1712.07642}.

\bibitem[Schaal(1997)]{schaal1997learning}
Stefan Schaal.
\newblock Learning from demonstration.
\newblock In \emph{Advances in neural information processing systems}, pages
  1040--1046, 1997.

\bibitem[Schulman et~al.(2015)Schulman, Levine, Moritz, Jordan, and
  Abbeel]{TRPO_algo}
John Schulman, Sergey Levine, Philipp Moritz, Michael~I. Jordan, and Pieter
  Abbeel.
\newblock Trust region policy optimization.
\newblock \emph{CoRR}, abs/1502.05477, 2015.
\newblock URL \url{http://arxiv.org/abs/1502.05477}.

\bibitem[Schulman et~al.(2017)Schulman, Wolski, Dhariwal, Radford, and
  Klimov]{PPO_algo}
John Schulman, Filip Wolski, Prafulla Dhariwal, Alec Radford, and Oleg Klimov.
\newblock Proximal policy optimization algorithms.
\newblock \emph{arXiv preprint arXiv:1707.06347}, 2017.

\bibitem[Sutton and Barto(2018)]{sutton2018reinforcement}
Richard~S Sutton and Andrew~G Barto.
\newblock \emph{Reinforcement learning: An introduction}.
\newblock MIT press, 2018.

\bibitem[Szegedy et~al.(2013)Szegedy, Zaremba, Sutskever, Bruna, Erhan,
  Goodfellow, and Fergus]{szegedy2013intriguing}
Christian Szegedy, Wojciech Zaremba, Ilya Sutskever, Joan Bruna, Dumitru Erhan,
  Ian Goodfellow, and Rob Fergus.
\newblock Intriguing properties of neural networks.
\newblock \emph{arXiv preprint arXiv:1312.6199}, 2013.

\bibitem[Tan et~al.(2018)Tan, Zhang, Coumans, Iscen, Bai, Hafner, Bohez, and
  Vanhoucke]{minitaur}
Jie Tan, Tingnan Zhang, Erwin Coumans, Atil Iscen, Yunfei Bai, Danijar Hafner,
  Steven Bohez, and Vincent Vanhoucke.
\newblock Sim-to-real: Learning agile locomotion for quadruped robots.
\newblock \emph{CoRR}, abs/1804.10332, 2018.
\newblock URL \url{http://arxiv.org/abs/1804.10332}.

\bibitem[Tobin et~al.(2017)Tobin, Fong, Ray, Schneider, Zaremba, and
  Abbeel]{tobin2017domain}
Josh Tobin, Rachel Fong, Alex Ray, Jonas Schneider, Wojciech Zaremba, and
  Pieter Abbeel.
\newblock Domain randomization for transferring deep neural networks from
  simulation to the real world.
\newblock In \emph{2017 IEEE/RSJ international conference on intelligent robots
  and systems (IROS)}, pages 23--30. IEEE, 2017.

\bibitem[Tobin et~al.(2018)Tobin, Biewald, Duan, Andrychowicz, Handa, Kumar,
  McGrew, Ray, Schneider, Welinder, et~al.]{tobin2018domain}
Josh Tobin, Lukas Biewald, Rocky Duan, Marcin Andrychowicz, Ankur Handa, Vikash
  Kumar, Bob McGrew, Alex Ray, Jonas Schneider, Peter Welinder, et~al.
\newblock Domain randomization and generative models for robotic grasping.
\newblock In \emph{2018 IEEE/RSJ International Conference on Intelligent Robots
  and Systems (IROS)}, pages 3482--3489. IEEE, 2018.

\bibitem[Todorov et~al.(2012)Todorov, Erez, and Tassa]{todorov2012mujoco}
Emanuel Todorov, Tom Erez, and Yuval Tassa.
\newblock Mujoco: A physics engine for model-based control.
\newblock In \emph{2012 IEEE/RSJ International Conference on Intelligent Robots
  and Systems}, pages 5026--5033. IEEE, 2012.

\bibitem[Torabi et~al.(2018{\natexlab{a}})Torabi, Warnell, and
  Stone]{torabi2018behavioral}
Faraz Torabi, Garrett Warnell, and Peter Stone.
\newblock Behavioral cloning from observation.
\newblock In \emph{Proceedings of the 27th International Joint Conference on
  Artificial Intelligence}, pages 4950--4957, 2018{\natexlab{a}}.

\bibitem[Torabi et~al.(2018{\natexlab{b}})Torabi, Warnell, and
  Stone]{torabi2018generative}
Faraz Torabi, Garrett Warnell, and Peter Stone.
\newblock Generative adversarial imitation from observation.
\newblock \emph{arXiv preprint arXiv:1807.06158}, 2018{\natexlab{b}}.

\bibitem[Torabi et~al.(2019)Torabi, Warnell, and Stone]{Torabi_2019}
Faraz Torabi, Garrett Warnell, and Peter Stone.
\newblock Recent advances in imitation learning from observation.
\newblock In \emph{Proceedings of the 28th International Joint Conference on
  Artificial Intelligence}, Aug 2019.

\bibitem[Xiao et~al.(2019)Xiao, Herman, Wagner, Ziesche, Etesami, and
  Linh]{xiao2019wasserstein}
Huang Xiao, Michael Herman, Joerg Wagner, Sebastian Ziesche, Jalal Etesami, and
  Thai~Hong Linh.
\newblock Wasserstein adversarial imitation learning.
\newblock \emph{arXiv preprint arXiv:1906.08113}, 2019.

\bibitem[Xie et~al.(2019)Xie, Clary, Dao, Morais, Hurst, and van~de
  Panne]{xie2019learning}
Zhaoming Xie, Patrick Clary, Jeremy Dao, Pedro Morais, Jonathan Hurst, and
  Michiel van~de Panne.
\newblock Learning locomotion skills for cassie: Iterative design and
  sim-to-real.
\newblock In \emph{Proc. Conference on Robot Learning (CORL 2019)}, volume~4,
  2019.

\bibitem[Yu et~al.(2018)Yu, Liu, and Turk]{policy-strategy-optimization}
Wenhao Yu, C.~Karen Liu, and Greg Turk.
\newblock Policy transfer with strategy optimization.
\newblock \emph{CoRR}, abs/1810.05751, 2018.
\newblock URL \url{http://arxiv.org/abs/1810.05751}.

\end{thebibliography}

% New page for appendices
\clearpage
\appendix

\section{Marginal Distributions and Returns} \label{app:marginal}

We expand the marginal transition distribution ($\rho_{s}$) definition to be more explicit below.

\begin{align}
    \rho_{sim,t}(s, a, s') &\defd \rho_{sim,t}(s) \p(a|s) \Tfun_{s}(s'|s, a)  \label{eqn:trans_t}\\
    \rho_{sim, t}(s') &\defd \sum_{s \in \sset} \sum_{a \in \aset} \rho_{sim, t-1}(s, a, s') \label{eqn:dist_t} \\
    \rho_{s}(s, a, s') &\defd (1 - \D) \sum_{t=0}^{\infty} \D^t \rho_{sim, t}(s, a, s') \label{eqn:trans_marg}
\end{align}

where $\rho_{sim,0}(s) = \rho_0(s)$ is the starting state distribution.
Written in a single equation:
\begin{align*}
    \rho_{s}(s, a, s') &= (1 - \D) \sum_{s_{0} \in \sset} \rho_0(s_{0}) \sum_{t=0}^{\infty} \D^{t} \sum_{a_{t}\in \aset} \sum_{s_{t+1} \in \sset} \p(a_{t}|s_{t}) \Tfun(s_{t+1}|s_{t}, a_{t})
\end{align*}

The expected return can be written more explicitly to show the dependence on the transition function.
It then makes the connection to \ref{eqn:marg_return} more explicit.
\begin{align*}
\mathbb{E}_{\p, \Tfun}\left[ G_0 \right] 
&= \mathbb{E}_{\p, \Tfun} \left[ \sum_{t=0}^{\infty} \D^t R(s_t, a_t, s_{t+1})\right] \\
&= \sum_{s_{0} \in \sset} \rho_0(s_{0}) \sum_{t=0}^{\infty} \D^{t} \sum_{a_{t}\in \aset} \sum_{s_{t+1} \in \sset} \p(a_{t}|s_{t}) \Tfun(s_{t+1}|s_{t}, a_{t}) \Rfun(s_{t}, a_{t}, s_{t+1})
\end{align*}

In the grounded source environment, the action transformer policy $\p_g$ transforms the transition function as specified in Section \ref{sec:s2r}.
Ideally, such a $\p_g \in \pset_g$ exists.
We denote the marginal transition distributions in sim and real by $\rho_{s}$ and $\rho_{t}$ respectively, and $\rho_g \in \mathcal{P}_{g}$ for the grounded source environment.
% \garrett{Perhaps this is too pedantic, but how did you decide when to use superscripts and when to use subscripts? I see it's subscripts for transition functions $P$, but superscripts for marginals $\rho$ -- but then a subscript is used for $\rho_0$ above. Is there a reason we can't just use subscripts for everything?}
% \ishan{the superscripts are partially leftover from when I was also including the agent policy $\p$ in the notation, and partially because $\rho_0$ indicates the initial distribution, and I didn't want to overload that term.}
% \garrett{I'd recommend changing everything to subscripts if possible.}
% 
The distribution $\rho_g$ relies on $\p_g \in \pset_g$ as follows:
\begin{align} \label{eqn:rho_g}
 \rho_g(s, a, s') &= (1 - \D) \p(a|s) \sum_{\Tilde{a} \in \aset} \Tfun_{s}(s'|s, \Tilde{a}) \p_g(\Tilde{a}|s, a) \sum_{t=0}^\infty \D^t p(s_t=s|\p, \Tfun_g)   
\end{align}

The marginal transition distribution of the source environment after action transformation, $\rho_{g}(s, a, s')$, differs in Equation \ref{eqn:trans_t} as follows:

\begin{align}
    \rho_{g,t}(s, a, s') &\defd \rho_{g,t}(s) \p(a|s) \sum_{\Tilde{a} \in \aset} \p_{g}(\Tilde{a}|s, a) \Tfun_{g}(s'|s, \Tilde{a})  \label{eqn:at_trans_t}
    % \rho_{t, at}(s') &\defd \D \sum_{s \in \sset} \sum_{a \in \aset} \rho_{t-1, at}(s, a, s') \label{eqn:at_dist_t} \\
    % \rho_{g}(s, a, s') &\defd \sum_{t=0}_{T} \rho_{t, at}(s, a, s') \label{eqn:at_trans_marg}
\end{align}

\section{Proofs} \label{app:proofs}

\subsection{Proof of Proposition \ref{prop:unique} }
\propunique*
\label{app:proof_unique}
\begin{proof}
We prove the above statement by contradiction.
Consider two transition functions $\Tfun_{1}$ and $\Tfun_{2}$ that have the same marginal distribution $\rho_\p$ under the same policy $\p$, but differ in their likelihood for at least one transition $(s, a, s')$. 
\begin{align}
    \Tfun_1(s'|s, a) \neq \Tfun_2(s'|s, a) \label{eqn:tneq}
\end{align}

Let us denote the marginal distributions for $\Tfun_{1}$ and $\Tfun_{2}$ under policy $\p$ as $\rho^\p_1$ and $\rho^\p_2$. Thus, $\rho^\p_1(s) = \rho^\p_2(s)$ $\forall s \in \sset$ and $\rho^\p_1(s, a, s') = \rho^\p_2(s, a, s') \forall s, s' \in \sset, a \in \aset$.

The marginal likelihood of the above transition for both $\Tfun_1$ and $\Tfun_2$ is:
\begin{align*}
    \rho^\p_1(s, a, s') &= \sum_{t=0}^{T-1} \rho^\p_1(s) \p(a|s) \Tfun_1(s'|s, a) \\
    \rho^\p_2(s, a, s') &= \sum_{t=0}^{T-1} \rho^\p_2(s) \p(a|s) \Tfun_2(s'|s, a)
\end{align*}

Since the marginal distributions match, and the policy is the same, this leads to the equality:
\begin{align}
    \Tfun_1(s'|s, a) = \Tfun_2(s'|s, a) \forall s, s' \in \sset, a \in \aset \label{eqn:teq}
\end{align}

Equation \ref{eqn:teq} contradicts Equation \ref{eqn:tneq}, proving our claim.
\end{proof}

\subsection{Proof of Proposition \ref{prop:opt}}
\label{app:opt}
\propopt*
\begin{proof}
We overload the notation slightly and refer to $\rho^{\p}_{t}$ as the marginal transition distribution in the target environment while following agent policy $\p$.
Proposition \ref{prop:unique} still holds under this expanded notation.

From Proposition \ref{prop:unique}, if $\Tfun_{t} = \Tfun_{g}$, we can say that $\rho^\p_{t} = \rho^\p_{g} \forall \p \in \pset$.
From Equation \ref{eqn:marg_return},  $\mathbb{E}_{\p,g} [G_0] = \mathbb{E}_{\p,real} [G_0]  \forall \p \in \pset$, and $\argmax_{\p \in \pset} \mathbb{E}_{\p,g} [G_0] = \argmax_{\p \in \pset} \mathbb{E}_{\p,real} [G_0]$.
\end{proof}

\subsection{Proof of Lemma \ref{lem:overline_equi}}
\label{app:lemoverline_equi_proof}
\lemoverlineequi*
\begin{proof}
For every $\rho_{g} \in \mathcal{P}_{g}$, there exists at least one action transformer policy $\p_{g} \in \pset_{g}$, from our definition of $\mathcal{P}_{g}$.
Let $\mathtt{RL} \circ \mathtt{ATIRL}_{\psi}(\Tfun_{t})$ lead to a policy $\Tilde{\p}_{g}$, with a marginal transition distribution $\Tilde{\rho}_{g}$.
The marginal transition distribution induced by $\overline{\mathtt{RL}} \circ \overline{\mathtt{ATIRL}}_{\psi}(\Tfun_{t})$ is $\overline{\rho}_{g}$.

We need to prove that $\Tilde{\rho}_{g} = \overline{\rho}_{g}$, and we do so by contradiction.
We assume that $\Tilde{\rho}_{g} \neq \overline{\rho}_{g}$.
For this inequality to be true, the marginal transition distribution of the result of $\mathtt{RL}(\Tilde{c})$ must be different than the result of $\overline{\mathtt{RL}}(\overline{c})$, or the cost functions $\Tilde{c}$ and $\overline{c}$ must be different.

Let us compare the $\mathtt{RL}$ procedures first.
Assume that $\Tilde{c} = \overline{c}$.
\begin{align*}
    \mathtt{RL}(\Tilde{c}) &= \argmin_{\p} \mathbb{E}_{\rho_g} \left[\Tilde{c}(s, a, s')\right] \\
    % &= \argmin_{\p} \sum_{s, a, s'} p(s|\p, \Tfun_g) \p(a|s) \Tfun(s'|s, a) \Tilde{c}(s, a, s') \\
    % &= \argmin_{\p} \sum_{s, a, s'} \p(a|s) \Tfun(s'|s, a) \Tilde{c}(s, a, s') (1 - \D) \sum_{t=0}^\infty \D^t p(s_t = s|\p, \Tfun_g) \\
    % &= \argmin_{\p} \sum_{s, a, s'} \rho_g(s, a, s') \Tilde{c}(s, a, s') \\
    &= \argmin_{\rho_g} \mathbb{E}_{\rho_g} \left[\Tilde{c}(s, a, s')\right] \texttt{  ...(surjective mapping)}\\
    &= \overline{\mathtt{RL}}(\overline{c}) \mathtt{  (\Tilde{c} = \overline{c})}
\end{align*}
which leads to a contradiction.

Now let's consider the cost functions presented by $\mathtt{ATIRL}_{\psi}(\Tfun_{t})$ and $\overline{\mathtt{ATIRL}}_{\psi}(\Tfun_{t})$.
Since $\mathtt{RL}(\Tilde{c})$ and $\overline{\mathtt{RL}}(\overline{c})$ lead to the same marginal transition distributions, for the inequality we assumed at the beginning of this proof to be true, $\mathtt{ATIRL}_{\psi}(\Tfun_{t})$ and $\overline{\mathtt{ATIRL}}_{\psi}(\Tfun_{t})$ must return different cost functions.

\begin{align*}
    \mathtt{ATIRL}_{\psi}(\Tfun_{t}) &= \argmax_{c \in \mathcal{C}} -\psi(c) + \left(\min_{\p_{g}} \mathbb{E}_{\Tfun_{g}}[c(s, a, s')] \right) - \mathbb{E}_{\Tfun_{t}}[c(s, a, s')] \\
    &=  \argmax_{c \in \mathcal{C}} -\psi(c) + \left(\min_{\p_{g}} \sum_{s, a, s'} \rho_g(s, a, s') c(s, a, s') \right) -\\& \sum_{s, a, s'} \rho_{t}(s, a, s')c(s, a, s') \\
    &=  \argmax_{c \in \mathcal{C}} -\psi(c) + \left(\min_{\rho_{g}} \sum_{s, a, s'} \rho_g(s, a, s') c(s, a, s') \right) -\\& \sum_{s, a, s'} \rho_{t}(s, a, s')c(s, a, s') \\
    &= \overline{\mathtt{ATIRL}}_{\psi}(\Tfun_{t})
\end{align*}

which leads to another contradiction.
Therefore, we can say that $\overline{\rho}_{g} = \rho_{\Tilde{g}}$.
\end{proof}

% \garrett{I didn't follow the last line ("Therefore, ...") at all. The argument I read is as follows:
% \begin{itemize}
%     \item For every $\rho^g$, we have a corresponding $\pi_g$ by definition. [I agree.]
%     \item Let the output of (unbarred composition) be $\tilde{\pi}_g$, which induces marginal $\tilde{\rho}^g$ (note I would switch the tilde to over $\rho$ instead of $g$ here). [OK.]
%     \item ???
%     \item Lemma proved!
% \end{itemize}
% I need something more to fill in the question marks. Even definitionally, I guess I'm supposed to assume $\bar{\rho}^g$ is the output of (barred composition)? More importantly, what's the glue that ties $\bar{\rho}^g$ and $\tilde{\rho}^g$ together? You need to be more explicit. I'm imagining it might be something like "assume they're not equal, then [something impossible happens]," but it could turn out different as well.
% }

\subsection{Proof of Lemma \ref{lem:overline_psi_equi}}
\label{app:lem_proof}

 We prove convexity under a particular agent policy $\p$ but across AT policies $\p_{g} \in \pset_{g}$

\begin{lemma} \label{lem:compact}
$\mathcal{P}_{g}$ is compact and convex.
\end{lemma}
\begin{proof}
We first prove convexity of $\rho_{\pset_{g}, t}$ for $\p_{g} \in \pset_{g}$ and $0 \leq t < \infty$, by means of induction.

Base case: $\lambda \rho_{at_1, 0} + (1 - \lambda) \rho_{at_2, 0} \in \rho_{\pset_{g},0}$, for $0 \leq \lambda \leq 1$.

\begin{align*}
    \lambda \rho_{at_1, 0}(s, a, s') + (1 - \lambda) \rho_{at_2, 0}(s, a, s') &= \lambda \rho_0(s)\p(a|s) \sum_{\Tilde{a} \in \aset} \p_{at_1}(\Tilde{a}|s, a) \Tfun_{s}(s'|s, \Tilde{a}) \\&+ (1 - \lambda) \rho_0(s)\p(a|s) \sum_{\Tilde{a} \in \aset} \p_{at_2}(\Tilde{a}|s, a) \Tfun_{s}(s'|s, \Tilde{a}) \\
    &= \rho_0(s)\p(a|s) \sum_{\Tilde{a} \in \aset} \left(\lambda \p_{at_1}(\Tilde{a}|s, a) + (1 - \lambda \p_{at_2}(\Tilde{a}|s, a))\right) \Tfun_{s} (s'|s, \Tilde{a})
\end{align*}

$\pset_{g}$ is convex and hence $\rho_0(s)\p(a|s) \sum_{\Tilde{a} \in \aset} \left(\lambda \p_{at_1}(\Tilde{a}|s, a) + (1 - \lambda \p_{at_2}(\Tilde{a}|s, a))\right) \Tfun_{s} (s'|s, \Tilde{a})$ is a valid distribution, meaning $\rho_{\pset_{g}, 0}$ is convex.

Induction Step: If $\rho_{\pset_{g}, t-1}$ is convex, $\rho_{\pset_{g}, t}$ is convex.

If $\rho_{\pset_{g},t-1}$ is convex, $\lambda \rho_{at_1,t}(s) + (1 - \lambda)\rho_{at_2,t}(s)$ is a valid distribution.
This is true simply by summing the distribution at time $t-1$ over states and actions.

% $\lambda \rho^\p_{t, at_1} + (1 - \lambda) \rho^\p_{t, at_2} \in \rho^\p_{t, \pset_{g}}$, for $0 \leq \lambda \leq 1$.

\begin{align*}
    \lambda \rho_{at_1,t}(s, a, s') + (1 - \lambda) \rho_{at_2, t}(s, a, s') &= \lambda \rho_{at_1, t}(s)\p(a|s) \sum_{\Tilde{a} \in \aset} \p_{at_1}(\Tilde{a}|s, a) \Tfun_{s}(s'|s, \Tilde{a}) \\&+ (1 - \lambda) \rho_{at_2,t}(s)\p(a|s) \sum_{\Tilde{a} \in \aset} \p_{at_2}(\Tilde{a}|s, a) \Tfun_{s}(s'|s, \Tilde{a}) \\
    &= \left(\lambda\rho_{at_1,t}(s) + (1 - \lambda) \rho_{at_2, t}(s) \right)\p(a|s)\\& \sum_{\Tilde{a} \in \aset} \left(\lambda \p_{at_1}(\Tilde{a}|s, a) + (1 - \lambda \p_{at_2}(\Tilde{a}|s, a))\right) \Tfun_{s} (s'|s, \Tilde{a})
\end{align*}

$\lambda \rho^\p_{at_1, t}(s) + (1 - \lambda)\rho^\p_{ at_1, t}(s)$ is a valid distribution, and $\pset_{g}$ is convex.
This proves that the transition distribution at each time step is convex.
The normalized discounted sum of convex sets (Equation \ref{eqn:trans_marg}) is also convex.
Since the exponential discounting factor $\D \in [0,1)$, the sum is bounded as well.
\end{proof}

We now prove Lemma \ref{lem:overline_psi_equi}.

\lemoverlinepsiequi*

\begin{proof}[Proof of Lemma \ref{lem:overline_psi_equi}]

Let $\overline{c} = \overline{\mathtt{ATIRL}}(\Tfun_{t})$,  $\overline{\rho}^{g} = \overline{\mathtt{RL}}(\overline{c}) = \overline{\mathtt{RL}} \circ \overline{\mathtt{ATIRL}}(\Tfun_{t})$ and

\begin{equation}
  \begin{aligned}
    \hat{\rho}_{g} = \argmin_{\rho_{g}} \psi^*(\rho_{g} - \rho_{t}) = \argmin_{\rho_{g}} \max_{c} -\psi(c) + \sum_{s, a, s'}& (\rho_{g}(s, a, s') \\&- \rho_{t}(s, a, s'))c(s, a, s')
  \end{aligned}
\end{equation}

where $\psi^*: \mathcal{C}^* \longmapsto \bar{\mathbb{R}}$ is the convex conjugate of $\psi$, defined as $\psi^*(c^*) \defd \sup_{c \in \mathcal{C}} \langle c^*, c\rangle - \psi(c)$.
Applying the above definition to the rightmost term in the above equation gives us the middle term.
% \garrett{How did you get the equality above? Is it purely the definition of $\psi^*$? Why are we suddenly summing over all $(s,a,s')$?}
% \ishan{Yes, it is the definition of the convex conjugate. We were summing over all $s, a, s'$ even in Equation \ref{eqn:bar_atirl}, from where we get this equation.}

We now argue that $\overline{\rho}_{g} = \hat{\rho}_{g}$ which are the two sides of the equation we want to prove.
Let us consider loss function $L: \mathcal{P}_{g} \times \mathbb{R}^{\sset \times \aset \times \sset} \longmapsto \mathbb{R}$ to be

\begin{align}
    L(\rho_{g}, c) = -\psi(c)  + \sum_{s, a, s'}(\rho_{g}(s, a, s') - \rho_{t}(s, a, s'))c(s, a, s')
\end{align}

We can then pose the above formulations as:
% \garrett{I think the middle equation below should be *arg* max, right? For the "max" with $c$ as the decision variable.}
\begin{align}
    \hat{\rho}_{g} &\in \argmin_{\rho_{g} \in \mathcal{P}_{g}} \max_c L(\rho_{g}, c) \label{eqn:hat_rho}\\
    \overline{c} &\in \argmax_c \min_{\rho_{g} \in \mathcal{P}_{g}} L(\rho_{g}, c) \label{eqn:bar_c}\\
    \overline{\rho}_{g} &\in \argmin_{\rho_{g} \in \mathcal{P}_{g}} L(\rho_{g}, \overline{c})
\end{align}

$\mathcal{P}_{g}$ is compact and convex (by Lemma \ref{lem:compact}) and $\mathbb{R}^{\sset \times \aset \times \sset}$ is convex. $L(\cdot, c)$ is convex over all $c$ and $L(\rho_{g}, \cdot)$ is concave over all $\rho_{g}$.
Therefore, based on minimax duality:

\begin{align}
     \min_{\rho_{g} \in \mathcal{P}_{g}} \max_c L(\rho_{g}, c) = \max_c \min_{\rho_{g} \in \mathcal{P}_{g}} L(\rho_{g}, c)
\end{align}

From Equations \ref{eqn:hat_rho} and \ref{eqn:bar_c}, $(\hat{\rho}_{g}, \overline{c})$ is a saddle point of $L$,
implying $\hat{\rho}_{g} = \argmin_{\rho_{g} \in \mathcal{P}_{g}} L(\rho_{g}, \overline{c})$ and so $\overline{\rho}_{g} = \hat{\rho}_{g}$.

\end{proof}

% \garrett{the above sentence doesn't read quite right. Is that the equation *implies that* $(\hat{\rho}_g,\bar{c})$ is a saddle point? And then this next fact? Also, here is where we're in trouble with $\in$ vs $=$. You say $\hat{\rho}_g \in$, but then $\hat{\rho}_g =$ next. Can we just switch to all $=$?}

\subsection{Proof of Lemma \ref{lem:psi_equi}} \label{app:proof_psi_equi}

\lempsiequivalence*
\begin{proof}
The proof of equivalence here is simply to prove that optimizing over $\p_g$ is the same as optimizing over $\rho_g$.
From Equation \ref{eqn:rho_g} and from the fact that agent policy $\p$ and source environment transition function $\Tfun_{s}$ are fixed, we can say that the only way to optimize $\rho_g$ is to optimize $\p_g$, which leads to the above equivalence.
\end{proof}

\section{Experimental Details}
\label{expt_details}

To collect expert trajectories from the target environment, we rollout the stochastic initial policy trained in sim for 1 million timesteps, on the target environment. This dataset serves as the expert dataset during the imitation learning step of \textsc{garat}. At each GAN iteration, we sample a batch of data from the grounded source environment and expert dataset and update the discriminator. Similarly, we rollout the action transformer policy in its environment and update $\pi_g$. We perform 50 such GAN updates to ground the source environment using \garat{}. The hyperparameters for the PPO algorithm used to update the action transformer policy is provided in Table \ref{ppo_hyperparams}. The hyperparameters used for the TRPO algorithm to update the agent policy can be found in Table \ref{trpo_hyperparams}.

We implemented different \ifo{} algorithms and noticed that there was no significant difference between these backend algorithms in sim-to-real performance. During the discriminator update step in \textsc{gai}f\textsc{o}-reverse\textsc{kl} (\textsc{airl}), \gaifo{} and \gaifo{}-\textsc{w} (\textsc{wail}), we use two regularizers in its loss function - $\mathtt{L2}$ regularization of the discriminator's weights and a gradient penalty (GP) term, with a coefficient of 10. Adding the GP term has been shown to be helpful in stabilizing GAN training \cite{GAN-gp}. 

In our implementation of the \textsc{airl} \cite{fu2018learning} algorithm, we do not use the special form of the discriminator, described in the paper, because our goal is to simply imitate the expert and does not require recovering the reward function as was the objective of that work. 
We instead use the approach \citet{ghasemipour2019divergence} use  with state-only version of \textsc{airl}.

\gat{} uses a smoothing parameter $\alpha$, which we set to $0.95$ as suggested by \citet{hanna2017grounded}.
\textsc{rarl} has a hyperparameter on the maximum action ratio allowed to the adversary, which measures how much the adversary can disrupt the agent's actions.
This hyperparameter is chosen by a coarse grid-search. For each domain, we choose the best result and report the average return over five policies trained with those hyperparameters.
%
% \ishan{Sid, check if this description for \textsc{rarl} is correct.}
%
We used the official implementation of \rarl{} provided by the authors for the MuJoCo environments. However, since their official code does not readily support PyBullet environments, for the Ant and Minitaur domain, we use our own implementation of \rarl{}, which we reimplemented to the best of our ability. 
When training a robust policy using Action space Noise Envelope (\textsc{ane}), we do not know the right amount of noise to inject into the agent's actions. Hence, in our analysis, we perform a sweep across zero mean gaussian noise with multiple standard deviation values and report the highest return achieved in the target environment with the best hyperparameter, averaged across 5 different random seeds. 

% \begin{table}
% \centering
%  \begin{tabular}{||c | c||} 
%  \hline
%  Name & Value \\ [0.5ex] 
%  \hline\hline
%  Hidden Layers & 2  \\ 
%  \hline
%  Hidden layer size & 64  \\
%  \hline
%  nminibatches & 2  \\ 
%  \hline
%  Num epochs & 1  \\ 
%  \hline
%  $\lambda$  & 0.95  \\ 
% \hline
%  $\gamma$ & 0.99  \\ 
%  \hline 
%  clipping ratio & 0.1  \\ 
%  \hline
%  time steps & 5000  \\  
%  \hline
% \end{tabular}
% \vskip 0.2in
% \caption{Hyperparameters for the TRPO algorithm used to update the Agent's Policy}
% \label{ppo_hyperparams}
% \end{table}

\begin{table}
\centering
 \begin{tabular}{||c | c||} 
 \hline
 Name & Value \\ [0.5ex] 
 \hline\hline
 Hidden Layers & 2  \\ 
 \hline
 Hidden layer size & 64  \\
 \hline
 timesteps per batch & 5000  \\ 
 \hline
 max KL constraint & 0.01  \\ 
 \hline
 $\lambda$  & 0.97  \\ 
\hline
 $\gamma$ & 0.995  \\ 
 \hline 
 learning rate & 0.0004  \\ 
 \hline
 cg damping & 0.1  \\  
 \hline
 cg iters & 20  \\  
 \hline
 value function step size & 0.001  \\  
 \hline
 value function iters & 5  \\  
 \hline
\end{tabular}
\vskip 0.2in
\caption{Hyperparameters for the TRPO algorithm used to update the Agent Policy}
\label{trpo_hyperparams}
\end{table}

\begin{table}
\centering
 \begin{tabular}{||c | c||} 
 \hline
 Name & Value \\ [0.5ex] 
 \hline\hline
 Hidden Layers & 2  \\ 
 \hline
 Hidden layer size & 64  \\
 \hline
 nminibatches & 2  \\ 
 \hline
 Num epochs & 1  \\ 
 \hline
 $\lambda$  & 0.95  \\ 
\hline
 $\gamma$ & 0.99  \\ 
 \hline 
 clipping ratio & 0.1  \\ 
 \hline
 time steps & 5000  \\  
 \hline
 learning rate & 0.0003  \\  
 \hline
\end{tabular}
\vskip 0.2in
\caption{Hyperparameters for the PPO algorithm used to update the Action Transformer Policy}
\label{ppo_hyperparams}
\end{table}

\subsection{Modified environments}
\label{environment_modifications}

\begin{table}
\centering
 \begin{tabular}{||c | c | c | c||} 
 \hline
 Environment Name & Property Modified & Default Value & Modified Value \\ [0.5ex] 
 \hline\hline
 InvertedPendulumHeavy & Pendulum mass & 4.89 & 100.0  \\ 
 \hline
 HopperHeavy & Torso Mass & 3.53 & 6.0  \\
 \hline
 HopperHighFriction & Foot Friction & 2.0 & 2.2  \\ 
 \hline
 HalfCheetahHeavy & Total Mass & 14 & 20  \\ 
 \hline
 WalkerHeavy & Torso Mass & 3.534 & 10.0 \\ 
 \hline
 Ant & Gravity & -4.91 & -9.81 \\ 
\hline
 Minitaur \cite{minitaur} & Torque vs. Current & linear & non-linear   \\ 
 \hline 
\end{tabular}
\vskip 0.2in
\caption{Details of the Modified target environments for benchmarking \textsc{garat} against other black-box transfer algorithms.}
\label{modified_envs_data}
\end{table}

We evaluate \textsc{garat} against several algorithms in the domains shown in Figure \ref{fig:results}. Table \ref{modified_envs_data} shows the source environment along with the specific properties of the environment/agent modified. We modified the values such that a policy trained in the sim environment is unable to achieve similar returns in the modified environment. By modifying an environment, we incur the risk that the environment may become too hard for the agent to solve. We ensure this is not the case by training a policy $\pi_{t}$ directly in the target environment and verifying that it solves the task.

\subsection{Source Environment Grounding Experimental Details} \label{app:sim_ground}

\begin{figure}[b]
    \centering
    \includegraphics[width=0.75\textwidth]{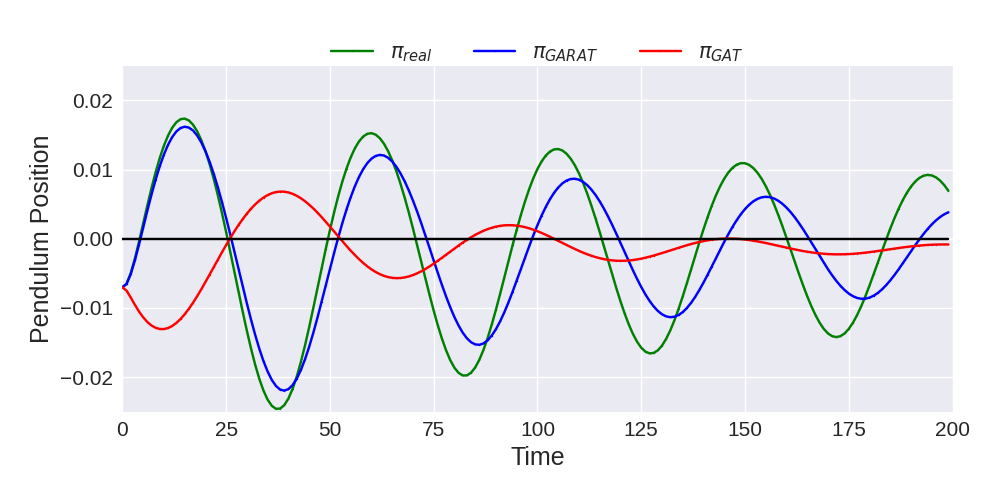}
    \caption{Policies trained in target environment, \gat{}-grounded source environment, and \garat{}-grounded source environment deployed in the target environment from the same starting state}
    \label{fig:real_pol}
\end{figure}

In Section \ref{sec:sim_ground}, we show results which validate our hypothesis that \garat{} learns an action transformation policy which grounds the source environment better than \gat{}.
Here we detail our experiments for Figure \ref{fig:trans_error}.

In Figure \ref{fig:persteptransitionerrors}, we plot the average error in transitions in source environments grounded with \garat{} and \gat{} with different amounts of target environment data, collected by deploying $\pi$ in the target environment.
The per step transition error is calculated by resetting the source environment state to states seen in the target environment,
taking the same action, and then measuring the error in the $\mathtt{L2}$-norm with respect to target environment transitions.
Figure \ref{fig:persteptransitionerrors} shows that with a single trajectory from the target environment, \textsc{garat} learns an action transformation that has similar average error in transitions compared to \textsc{gat} with $100$ trajectories of target environment data to learn from. 

In Figure \ref{fig:trajs}, we compare \garat{} and \gat{} more qualitatively.
We deploy the agent policy $\pi$ from the same start state in the target environment, the source environment, \gat{}-grounded source environment, and \garat{}-grounded source environment.
Their resultant trajectories in one of the domain features (angular position of the pendulum) is plotted in Figure \ref{fig:trajs}.
The trajectories in \garat{}-grounded source environment keeps close to the target environment, which neither the ungrounded source environment nor the \gat{}-grounded source environment manage.
The trajectory in the \gat{}-grounded source environment can be seen close to the one in the target environment initially, but since it disregards the sequential nature of the problem, the compounding errors cause the episode to terminate prematurely.

An additional experiment we conducted was to compare the policies trained in the target environment, \gat{}-grounded source environment and \garat{}-grounded source environment.
This comparison is done by deploying them in the target environment from the same initial state.
As we can see in Figure \ref{fig:real_pol}, the policies trained in the target environment and the \garat{}-grounded source environment behave similarly, while the one trained in the \gat{}-grounded source environment acts differently.
This comparison is another qualitative one.
How well these policies perform in w.r.t. the task at hand is explored in detail in Section  \ref{sec:exp_transfer}.

\end{document}